%% file: CDCP_TPAMI.tex
\documentclass[lettersize,journal]{IEEEtran}
\usepackage{amsmath,amsfonts}
\usepackage{array}
\usepackage{textcomp}
\usepackage{stfloats}
\usepackage{url}
\usepackage{verbatim}
\usepackage{graphicx}
\usepackage{cite}

\usepackage{multicol}
\usepackage{multirow}
\usepackage{mathtools} 
\usepackage{amsmath}
\usepackage{bm}
\usepackage{amssymb}
\usepackage{mathtools} 
\usepackage{amsthm}   %
\usepackage{multirow}
\usepackage{anyfontsize}
\usepackage{amsfonts}       
\usepackage{multicol}
\usepackage{multirow}
\usepackage{graphicx}

\newtheorem{theorem}{Theorem}[section]
\newtheorem{remark}[theorem]{Remark}
\newtheorem{proposition}[theorem]{Proposition}
\newtheorem{lemma}[theorem]{Lemma}

\usepackage[linesnumbered,ruled,vlined]{algorithm2e}
\usepackage{amsmath} 
\usepackage{dsfont}
\usepackage{subcaption}
\usepackage{booktabs}
\usepackage{colortbl}
\usepackage{xcolor}

\begin{document}

\title{CDCP: Conditional Diffusion Model with Contextual Prompts for Multi-task Offline Safe Reinforcement Learning}

\author{
        Jiayi Guan$^{1}$,
        Tianle Zhang$^{2}$,
        Li Shen$^{3}$,
        Ruiqi Zhang$^{1, 4}$,
        Ao Zhou$^{1}$,
        Lusong Li$^{2}$, 
        ~Guang Chen$^{1, 5,*}$,~\IEEEmembership{Member,~IEEE},\\
        Alois Knoll$^{5}$,~\IEEEmembership{Fellow,~IEEE},     
        Xiaodong He$^{2}$,~\IEEEmembership{Fellow,~IEEE},           
        and Changjun Jiang$^{1}$ 
\thanks{$^{*}$Corresponding author: Guang Chen (mail to: guangchen@tongji.edu.cn)}
\thanks{
Guang Chen, Jiayi Guan, Ruiqi Zhang, and Changjun Jiang are with Tongji University, Shanghai, China}
\thanks{Tianle Zhang, Lusong Li, and Xiaodong He are with JD Explore Academy, Beijing, China}
\thanks{
Li Shen is with the School of Cyber Science and Technology, Shenzhen Campus of Sun Yat-sen University, Shenzhen, China.}
\thanks{
Ruiqi Zhang is with the Berkeley AI Research Lab, University of California, Berkeley, United States.}
\thanks{
Guang Chen and Alois Knoll are with the Technical University of Munich, Germany.
}
}

\markboth{~}%
{Shell \MakeLowercase{\textit{et al.}}: A Sample Article Using IEEEtran.cls for IEEE Journals}


\maketitle


\input{sec/0_abstract}


\vspace{-0.2cm}
\section{Introduction}
\label{sec:introduction}
\IEEEPARstart{M} \noindent {ulti-task} reinforcement learning~(RL) promises to find a shared optimal policy that maximizes the rewards across all tasks~\cite{tomov2021multi,huang2022curriculum, NEURIPS2020_32cfdce9}. This paradigm is crucial for achieving general-purpose decision-making or controls~\cite{yu2020meta,xie2022lifelong,li2020multi, guan2022discrete, nie2018longitudinal}, and it has made remarkable achievements in fields such as robotic manipulation~\cite{NEURIPS2022_f376f5df, ze2023gnfactor}, autonomous driving~\cite{ishihara2021multi, parvini2023aoi}, and industrial applications~\cite{chen2021multi,hessel2019multi}. Although these advancements are notable, most previous successful cases are limited to maximizing the average reward across all tasks~\cite{vithayathil2020survey,zhang2021provably}. However, in real-world applications, it is not only necessary to maximize the rewards for all tasks but also to ensure that the agent meets safety constraints. For instance, in autonomous driving, it is essential to maximize speed and smoothness while avoiding dangerous driving and collisions~\cite{he2024toward, zhang2022residual, li2022metadrive}. Similarly, in robotic manipulation, the agent must efficiently and accurately perform tasks such as pushing, picking, grasping, and placing objects, while ensuring that it does not collide dangerously with the surrounding environment or objects~\cite{adjei2024safe,wachi2024safe}. To this end, multi-task offline safe RL aims to learn an optimal safe policy from offline data across multiple tasks. This paradigm ensures the safety of multi-task RL algorithms during training and deployment by utilizing offline learning and safety constraints. Consequently, it is crucial for policy optimization in multi-task scenarios with interaction risks.



Although multi-task offline safe RL algorithms are crucial for scenarios with high interaction costs and risks, this field still lacks sufficient research. Moreover, the triple difficulties of multi-tasks, safety constraints, and out-of-distribution~(OOD) actions present a new challenge for optimizing multi-task offline safe RL. On the one hand, within the paradigm of the offline safe RL, some previous work~\cite{xu2022constraints,guan2024poce} addresses the impact of OOD actions by reducing the extrapolation error of Q-values through pessimistic conservative estimation. Additionally, other studies~\cite{lee2022coptidice, polosky2022constrained} attempt to minimize the influence of OOD actions by leveraging stationary distributions. These methods have shown promising results in offline safe RL for single tasks. However, when extended to multi-task settings, the involvement of multiple tasks' Q-values and variational distributions can easily lead to the accumulation of estimation errors in conservative estimation and stationary distributions. The accumulation of errors makes it challenging for the algorithm to learn a policy that maximizes rewards based on Q-values. To address safety constraints, some existing work~\cite{guan2023voce} converts the safety constraint problem in offline safe RL tasks into an unconstrained problem using methods such as the Lagrangian multiplier or penalty function. These methods rely on accurate Q-values. When applied to the domain of multi-task safe RL, the estimation errors in the cost Q-values across multiple tasks make it challenging for the algorithm to guarantee the safety of the policy.


On the other hand, within the paradigm of multi-task RL, several studies have attempted to enhance the efficiency of simultaneous learning across multiple tasks through knowledge sharing~\cite{d2019sharing,varghese2021optimization,yu2021conservative} and contextual representation~\cite{sodhani2021multi}. Building on this foundation, some works~\cite{chen2020just,liu2021conflict} have adopted gradient normalization methods to adjust the combined gradient, mitigating interference among tasks during training. Additionally, sparse matrix techniques~\cite{chen2020just,d2019sharing} and parameter subspace methods~\cite{yang2020multi,sun2022paco,sarafian2021recomposing} have efficiently represented multiple tasks within limited parameter constraints. While these studies have significantly enhanced multi-task learning efficiency, mitigating task interference, and achieving efficient multi-task representation, most existing approaches rely on Q-value evaluations for policy updates. However, Q-values are sensitive to OOD actions and necessitate multiple-cost Q-values in multi-task offline safe RL tasks. This makes it difficult for existing multi-task RL algorithms to meet safety constraints and achieve the expected reward returns when extended to the multi-task offline safe RL paradigm. To summarize, while previous works have made notable progress in single-task and unconstrained multi-task paradigms, tackling the combined challenges of multi-task interference, OOD actions, and safety constraints in the multi-task offline safe RL remains a significant hurdle.

To address the challenges above, we propose a conditional diffusion model with contextual prompts for the multi-task offline safe RL. Concretely, we establish the objective for multi-task offline safe RL and transform the constrained optimization problem into an optimal safe action generation problem using the conditional diffusion model. Building on this, we design a novel cost-constraint strategy with classifier-free guidance, eliminating extrapolation errors caused by OOD actions through supervised learning, effectively addressing the cost-constraint issue in multi-task settings. Additionally, we propose a contextual prompting method to tackle multi-task representation problems. Furthermore, we introduce a gradient loss synchronization strategy that eliminates gradient interference among tasks by adaptively adjusting the weights of gradient losses. Finally, we conduct extensive comparative and ablation experiments on the CDCP in the \textit{MetaDrive} environment.
The main contributions are listed as follows:
\begin{itemize}
    \vspace{-0.1cm}
    \item To the best of our knowledge, we are the first to establish the objective of multi-task offline safe RL to solve the problem of learning a safe policy from offline data across multiple tasks. Moreover, we transform the constrained optimization problem into a conditional generation problem using a conditional diffusion model.
    \item We design a novel cost constraint strategy based on classifier-free guidance that eliminates extrapolation errors caused by OOD actions through supervised learning and generates safe action sequences for multiple tasks based on trajectory cost and reward conditions.
    \item We propose an innovative contextual prompting method to tackle the multi-task representation problem. This method enhances task representation accuracy by integrating textual descriptions with trajectory prompts while improving the algorithm's adaptability to unseen tasks.
    \item We introduce a gradient loss synchronization strategy that adaptively adjusts the weights of each task's gradient loss based on the gradient loss norm and decay rate. It eliminates gradient interference among tasks, thereby enhancing the stability of the algorithm's training process.
    \item Extensive comparative and ablation experiments demonstrate that the CDCP achieves higher reward returns than the previous state-of-the-art algorithms while ensuring safety. Furthermore, the experiments show that the CDCP provides a more flexible cost-constrained strategy.
\end{itemize}

The remainder of this paper is organized as follows. Section II reviews the related literature, covering offline safe RL, multi-task offline RL, and diffusion models. Section III provides background information on multi-task offline safe IV and rethinks the core requirements in real-world scenarios in multi-task domains. In Section V, we introduce the CDCP algorithm and its practical implementation. Section VI presents the experiments and analyzes the results. Finally, the conclusion is drawn in Section VII

\input{sec/4_related_work}

\input{sec/2_preliminaries}
\input{sec/3_method}
\input{sec/5_experimental}

\input{sec/6_conclusion}


\section*{Acknowledgments}
This work is supported by the National Natural Science Foundation of China (No. 62372329), in part by the National Key Research and Development Program of China (No.2021YFB2501104), in part by Shanghai Rising Star Program (No.21QC1400900), in part by Tongji-Qomolo Autonomous Driving Commercial Vehicle Joint Lab Project, and in part by Xiaomi Young Talents Program. I would like to thank Dr. Long Yang for his insightful discussions on the methodology.



\bibliographystyle{IEEEtran}
\bibliography{ref.bib}

\input{sec/7_supplement}


\vfill

\end{document}

%% file: sec/0_abstract.tex
\vspace{-0.4cm}
\begin{abstract}

Multi-task offline safe reinforcement learning~(RL) promises to learn a shared optimal safe policy from offline data across multiple tasks. This paradigm provides an effective means for the widespread application of RL in multi-task scenarios with high risk and interaction costs. However, the triple challenges of multi-tasking, safety constraints, and out-of-distribution~(OOD) actions pose a significant hurdle for existing methods to ensure safety while maximizing reward returns. In this work, we propose a Conditional Diffusion model with Contextual Prompts~(CDCP) to address these challenges. Concretely, we first rethink the requirements and challenges in current multi-task decision-making and control scenarios and establish the objectives of multi-task offline safe RL. Subsequently, we transform the multi-task constrained optimization problem into a conditional generation problem using the diffusion model. Based on this, we design a classifier-free guided cost-constraint strategy to provide flexible cost constraints and eliminate extrapolation errors from OOD actions via supervised learning. Additionally, we introduce a novel contextual prompting method to enhance multi-task representation accuracy and adaptability to unseen tasks. A gradient loss synchronization strategy is also introduced to eliminate gradient interference, improving training stability. Finally, extensive experiments demonstrate that the CDCP algorithm exhibits higher performance and safety in multi-task scenarios than the current state-of-the-art baseline methods. It meets different cost constraints without further training, providing a more flexible cost-constraint solution for the multi-task safe RL.


\end{abstract}


%% file: sec/4_related_work.tex
\section{Related Work}

In this section, we extensively discuss relevant works on multi-tasks offline safe RL. We mainly focus on two aspects: offline safe RL and multi-task RL.

\vspace{-0.4cm}
\subsection{Offline Safe RL}
Offline safe reinforcement learning optimizes policies that meet safety constraints using offline static data without further environmental interaction. This paradigm not only ensures the safety of policy deployment processes but also assures the safety of the policy training process~\cite{le2019batch,guan2023voce}. Consequently, it has attracted significant attention from scholars. After the proposal of the constraints penalized Q-learning algorithm (CPQ)~\cite{xu2022constraints}, numerous scholars have endeavored to address the optimization problem of safety policies under such offline settings by employing diverse methodologies. The VOCE~\cite{guan2023voce} algorithm combines variational inference with pessimistic conservative estimation to optimize safety-constrained policies from offline data. It shows promising results across samples with varying safety. Subsequently, the POCE~\cite{guan2024poce} algorithm addresses the optimization problem of RL policies with multiple safety constraints in offline settings by introducing a primal policy optimization method and proposing a conditional Bellman operator. Additionally, some studies attempt to address the optimization of safe policies from offline data using dynamic programming techniques from the perspective of stationary distributions~\cite{zhang2019gendice,dai2020coindice,lee2021optidice}. Among these, both COPO~\cite{polosky2022constrained} and CoptiDICE~\cite{lee2021coptidice} achieve the optimization of safe RL policies by solving for the stationary distribution of the optimal policy and transforming the constrained optimization problem into an unconstrained one using the Lagrange multiplier method. On the other hand, distinct from traditional temporal difference methods, some studies address the problem of constrained safety policy optimization in offline settings from a sequential decision-making perspective. For instance, both CDT~\cite{liu2023constrained} and Saformer~\cite{zhang2023saformer} utilize the transformer framework to reformulate the optimization of safety policies as a causal reasoning-based safe sequence decision-making problem, achieving promising results in various environments.
\vspace{-0.6cm}
\subsection{Multi-task RL}
Multi-task RL aims to learn a shared policy for multiple tasks. In practical scenarios, robots~\cite{kalashnikov2021scaling, shridhar2023perceiver} and autonomous vehicles~\cite{ishihara2021multi, liang2022effective} are required to manage a diverse range of control tasks and environments, underscoring the critical importance of multi-task RL for advancements in robotics and autonomous driving. To enhance the efficiency and stability of multi-task RL policy learning, some studies~\cite{borsa2016learning,li2020multi,cheng2023multi,varghese2021optimization,deramo2020sharing} address the problem of shared policy learning across multiple tasks through knowledge sharing or task representation methods. Notably, IMPALA~\cite{espeholt2018impala} and Multi-Goal RL~\cite{plappert2018multi}, inspired by the exploration evaluation framework, build on single-task RL algorithms such as A3C~\cite{mnih2016asynchronous} and DDPG~\cite{silver2014deterministic}. They achieve shared policy learning across multiple tasks by leveraging knowledge sharing to train multiple Q-value evaluation networks. Subsequently, CDS~\cite{yu2021conservative} refines data routing based on task-specific data, thereby implementing a conservative data-sharing strategy grounded in Q-value estimation. Additionally, some studies focus on enhancing multi-task RL policy training and improving model representational capacity by addressing gradient interference and parameter subspaces. In particular, CAGrad~\cite{liu2021conflict} and GradDrop~\cite{chen2020just} employ regularization methods to alleviate the maximum local loss for each task or utilize consistency strategies in sampling gradients at activation layers to adjust losses, thereby reducing gradient interference between tasks. To further enhance the multi-task representation capacity of the model, the PACO~\cite{sun2022paco} and MFQI~\cite{d2020sharing} algorithms learn an additional policy subspace or sparse matrix to represent multi-task policies. On the other hand, recent studies such as MDT~\cite{lee2022multi} and MTDiff~\cite{he2023diffusion} extend sequence modeling methods designed for single tasks to the multi-task RL domain, achieving promising results in various simulation scenarios.

In summary, our work differs from existing studies in two main aspects, representing our core challenges. First, we address the problem of optimizing a shared safety policy using only offline data from multiple tasks. Second, we employ a novel conditional generation approach instead of the temporal difference method to achieve multi-task safety policy optimization. To the best of our knowledge, this is the first attempt to optimize multi-task safety policies using offline data.

%% file: sec/2_preliminaries.tex
\section{Preliminaries}
In this section, we first introduce the fundamental concepts of offline safe RL, multi-task offline RL, and diffusion model. Subsequently, we rethink the core requirements and existing challenges in multi-task domains within real-world scenarios, clarifying the purpose and significance of this work.
\subsection{Offline Safe RL}
Constrained Markov Decision Processes~(CMDP) provide a theoretical framework to solve safe RL problems~\cite{cmdp}. It is defined as a tuple $(\mathcal{S},\mathcal{A}, \mathcal{P}, c, r, \rho_{0},\gamma)$, where $\mathcal{S}\in \mathbb{R}^{n}$ is the state space, $\mathcal{A} \in \mathbb{R}^m$ is the action space, $\mathcal{P}:\mathcal{S}\times\mathcal{A}\times\mathcal{S}\rightarrow[0,1]$ is the transition kernel, which specifies the transition probability $p(s_{t+1}|s_t,a_t)$ from state $s_t$ to state $s_{t+1}$ under the action $a_t$, ${r}:\mathcal{S}\times\mathcal{A}\rightarrow\mathbb{R}$ represents the reward function, $c$ is the set of costs $\{{c_{i}:\mathcal{S}\times\mathcal{A}\rightarrow \mathbb{R_{+}}},i=1,2,\cdots,m\}$ for violating $m$ constraints, $\gamma\in[0,1)$ is the discount factor, and $\rho_{0}:\mathcal{S}\rightarrow [0,1]$ is the distribution of initial states. The policy $\pi$ is a probability distribution mapping the state $s_t$ to the action $a_t$. We use shorthand $r_{t} =r(s_t,a_t)$ and $c_{i,t}=c_{i}(s_t,a_t)$ for simplicity. The common objective of constrained RL is to maximize the cumulative reward while satisfying the cumulative cost constraint.
\small
\begin{equation} 
    \label{equ0_1}
    \begin{aligned}     
     \!\!\!\pi^{*}\!=\!\arg \max_{\pi}\mathbb{E}_{\rho_{0},\pi}\left[\sum_{t=0}^{\infty}\gamma^{t}r_{t}\right], ~
     \mathrm{s.t.} ~ \mathbb{E}_{\rho_{0},\pi}\left[\sum_{t=0}^{\infty}\gamma^{t}c_{i,t}\right]\leq \bar{c}_i,
     \end{aligned}
\end{equation}\normalsize
\noindent where the $\bar{c}_i$ is the cost threshold of the $i$-th cumulative cost constraint.
Although the aforementioned online-constrained RL methods can ensure the agent's safety during the testing and deployment phases, their safety during the training process faces challenges. Offline safe RL is a method of learning a policy that satisfies cost constraints from offline datasets without interacting with the environment~\cite{polosky2022constrained, lee2021optidice}. This paradigm not only addresses the safety concerns during the testing and deployment phases but also ensures the safety of the training process.

\subsection{Multi-task Offline RL}
In multi-task RL, different tasks may possess distinct reward functions, state spaces, and transition matrices. Given a task $\tau\sim\mathcal{T}$, where $\mathcal{T}$ is a set of tasks, multi-task RL is typically defined based on the standard RL framework as $\{(\mathcal{S}^{\tau}, \mathcal{A}^{\tau},\mathcal{P}^{\mathcal{\tau}},r^{\tau},\rho_{0}^{\tau},\gamma)\}_{\tau}^{\mathcal{T}}$. The objective of multi-task RL differs from that of standard RL. It aims to find an optimal policy that maximizes the average reward across all tasks rather than the reward for a single task. Therefore, our multi-task RL objective is typically defined as:
\small
\begin{equation} 
    \label{equ3_2}
    \begin{aligned}     
     \pi^{*}=\arg \max_{\pi}\mathbb{E}_{\tau\sim\mathcal{T}}\mathbb{E}_{\rho_{0},\pi}\left[\sum_{t=0}^{\infty}\gamma^{t}r_{t}^{\tau}\right].
     \end{aligned}
\end{equation}\normalsize

In the offline setting, only the offline static datasets are accessible during training. Therefore, offline multi-task RL aims to learn a shared policy that maximizes the reward across all tasks using a static dataset collected by the unknown behavior policy $\pi_{\beta}^{\tau}$. Existing offline RL algorithms typically update policies by maximizing Q-values. However, OOD actions can easily lead to extrapolation errors in Q-value estimation in the offline setting~\cite{kumar2020conservative,guan2023uac}. Because multi-task offline RL involves the estimation of multiple Q-values, these extrapolation errors are more pronounced in multi-task scenarios. Therefore, addressing extrapolation errors caused by OOD actions is a significant challenge for multi-task RL.

\subsection{Diffusion Model}
Diffusion models are a typical type of generative model~\cite{song2019generative, ho2020denoising}. The data generation procedure in diffusion models is modeled as a predefined forward noise and a trainable reverse denoising process. The forward diffusion process follows a Markov chain, using a predefined variance schedule $\{\sigma_{1}, \sigma_{2},\cdots,\sigma_{K}\}$ to gradually add Gaussian noise to the samples, producing a series of noisy samples $\{x_{1},x_{2},\cdots, x_{K}\}$.
The forward diffusion process is defined as:
\vspace{-0.05cm}
\begin{equation} 
    \begin{aligned}
    &q(x_{k}|x_{k-1}) = \mathcal{N}(x_{k};\sqrt{1-\sigma_{k}}x_{k-1},\sigma_{k}I),\\
    \vspace{-0.05cm}
    &\quad \quad q(x_{1:K}|x_{0}) = \prod_{k=1}^{K}q(x_{k}|x_{k-1}),
    \end{aligned}
\end{equation}
\normalsize
\noindent where $K$ is the maximum number of diffusion steps. The reverse denoising process starts with Gaussian noise and gradually denoises to generate new samples by learning a trainable conditional distribution $p_{\theta}(x_{k-1}|x_{k})$. The denoising process is defined as:
\small
\vspace{-0.05cm}
\begin{equation} 
    \begin{aligned}
    & p_{\theta}(x_{k-1}|x_{k})=\mathcal{N}(x_{k-1};\mu(x_{k},k),\Sigma(x_{k},k)),\\
    \vspace{-0.05cm}
    & \quad \quad p_{\theta}(x_{0:K}) = p(x_{K})\prod_{k=1}^{K}p_{\theta}(x_{k-1}|x_{k}),   
    \vspace{-0.2cm}
    \end{aligned}
\end{equation}
\normalsize
\noindent where $p(x_{K})\!\! = \!\!\mathcal{N}(0,I)$ under the condition that $\prod_{k=1}^{K}(\!1-\!\sigma_{k}) \!\approx \! 0 $. Although diffusion models can be trained by optimizing the variational low bound of $\log p_{\theta}$~\cite{blei2017variational,kang2024efficient}, they ~\cite{ajay2022conditional,kang2024efficient} are generally trained using a simplified surrogate loss instead:
\begin{equation} 
    \begin{aligned}
        \label{equ_46}
        &\mathcal{L}(\theta)=\mathbb{E}_{k,\epsilon}\Big[\big{\Vert}\epsilon-\epsilon_{\theta}\big(x_{k}, k \big)\big{\Vert}^{2}\Big],\\  
    \end{aligned}
\end{equation}
\noindent where $\epsilon\!\sim\!\mathcal{N}(0,I)$ is Gaussian noise and $\epsilon_{\theta}$ represents the parameterized predicted noise. Given that $\mu(x_k,k)$ can be estimated from $\epsilon_{\theta}$~\cite{ho2020denoising}, this is equivalent to predicting the mean of $p_{\theta}$.

\subsection{Rethinking the requirements and challenges in Multi-task}
\label{pre: rethink}

We rethink the core requirements of multi-task reinforcement learning in real-world scenarios. In autonomous driving, we aim not only to maximize driving speed and smoothness across various traffic scenarios but also to mitigate risky driving behaviors and prevent traffic accidents. Similarly, in the domain of robotics manipulation, we not only need to consider efficiently completing various manipulation tasks to maximize rewards but also need to avoid risky collision behaviors. These tasks require maximizing the overall rewards of all tasks and considering the cost constraints for each task.

\begin{figure}[htp!]
    \centering
    \vspace{-0.2cm}
    \includegraphics[scale=0.285]{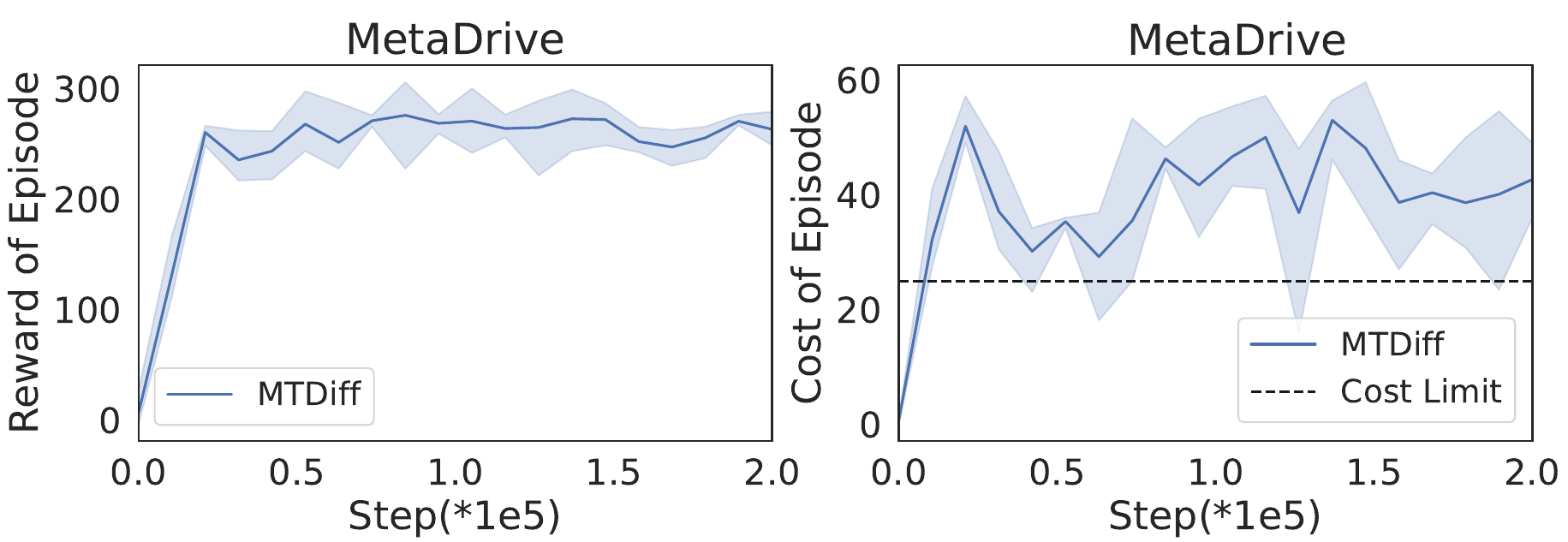}
    \caption{The mean reward and cost for the multi-task offline RL algorithm MTDiff across nine driving tasks, each with varying traffic scenarios or densities. The curves represent the test results from 3 random seeds during training, with the solid line indicating the mean of the three random seeds and the shaded area representing the standard deviation.}
    \label{fig:pre_mtrl}
\end{figure}
\begin{figure}[htp!]
    \centering
    \vspace{-0.30cm}
    \includegraphics[scale=0.25]{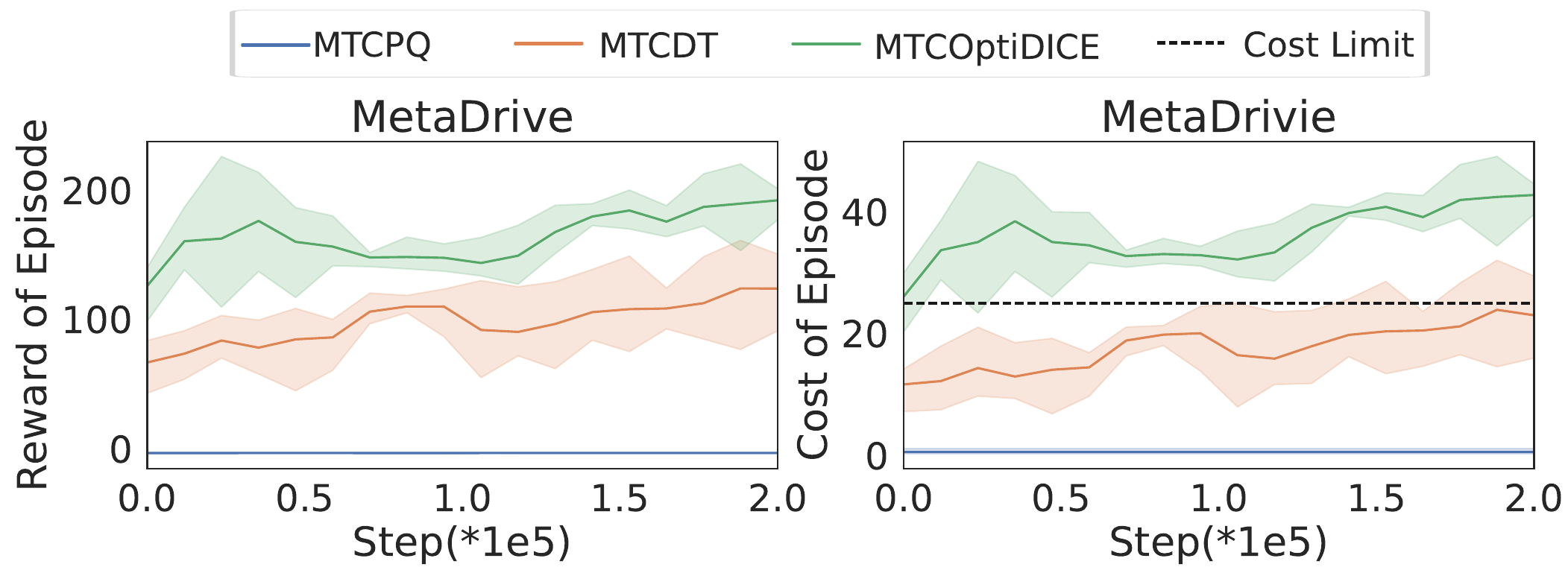}
    \caption{The average reward and cost of the algorithms that extend single-task offline safe RL to the multi-task offline safe RL. These results are across nine driving tasks with varying traffic scenarios or densities. The curves represent the test results from 3 random seeds during training, with the solid line indicating the mean of the 3 random seeds and the shaded area representing the standard deviation.}
    \vspace{-0.15cm}
    \label{fig:pre_osrl}
\end{figure}

We conduct experiments on state-of-the-art multi-task RL algorithms in the typical multi-task safety RL experimental environment \textit{MetaDrive}~\cite{li2022metadrive,li2021efficient,peng2022safe}. As shown in Fig~\ref{fig:pre_mtrl}, the experimental results indicate that although existing offline multi-task RL algorithms achieve high rewards, the costs far exceed the acceptable range for practical applications. The results suggest that while existing multi-task offline RL algorithms can maximize the rewards of each task, they struggle to ensure that the agent operates within safe constraints. Additionally, we extend state-of-the-art single-task offline safe RL algorithms to the multi-task offline safe RL domain and conduct experiments in the \textit{MetaDrive} environment. The experimental results, shown in Fig.~\ref{fig:pre_osrl}, indicate that although some algorithms adapted from single-task to multi-task domains can ensure costs meet the constraints, the reward returns are significantly reduced. Moreover, some algorithms still exhibit costs that far exceed the cost threshold. The results and analysis above indicate that while existing multi-task offline RL algorithms can maximize the reward return for each task, they struggle to ensure that costs meet the constraints. Furthermore, although some multi-task offline safe RL algorithms extended from single-task methods can satisfy cost constraints, their reward returns are significantly lower than expected, and some algorithms still fail to ensure cost constraints.

In summary, existing multi-task offline RL algorithms struggle to meet safety constraints in multi-task scenarios. Furthermore, while some multi-task offline safe RL algorithms extended from single-task methods can ensure safety constraints, their reward returns are significantly lower than expected. Therefore, we propose a novel paradigm of multi-task offline safe RL to address the challenge of learning a shared safety policy from offline in multi-task scenarios.

%% file: sec/3_method.tex
\section{Method}
This section elaborates on the conditional diffusion model with contextual prompts algorithm for multi-task offline safe RL. Initially, we define the objectives of multi-task offline safe RL and design a conditional generation model to optimize safety policies across multiple tasks. Subsequently, we propose a contextual prompting method to tackle task representation issues. Building on this, we introduce a gradient loss synchronization strategy to mitigate gradient interference among multiple tasks. Finally, we demonstrate the algorithm's application process through practical instances.



\subsection{Problem Setup}
\label{cmtrl}

Based on the discussion and analysis in Section~\ref{pre: rethink}, it is evident that existing multi-task offline safe RL algorithms face challenges in learning a shared safety policy from offline datasets of multiple tasks. To address this challenge, we propose a multi-task offline safe RL paradigm aimed at learning a policy that satisfies safety constraints while maximizing the average reward across all tasks from multiple offline datasets. Considering that multi-task offline safe RL aims to maximize the average reward across all tasks while ensuring that the costs for all tasks meet safety constraints, we define the multi-task RL paradigm as presented in Definition~\ref{def-1}.

\newtheorem{myDef}{\bf{Definition}}                                                             
\begin{myDef}
    \label{def-1} 
    \textit{(MTOS) Given a task $\tau\sim\mathcal{T}$, where $\mathcal{T}$ is a set of tasks, Multi-task Offline Safe RL~(MTOS) is defined as a tuple $\{(\mathcal{S}^{\tau},\mathcal{A}^{\tau}, \mathcal{P}^{\tau}, r^{\tau}, c_{i}^{\tau},\rho_{0}^{\tau},\gamma)\}_{\tau}^{\mathcal{T}}$ with the objective written as:}
    \begin{equation} 
        \label{equ3_3}
        \begin{aligned}
        & \pi^{*}=\arg\max_{\pi}\mathbb{E}_{\tau\sim\mathcal{T}}\mathbb{E}_{\rho_{0},\pi}\left[\sum_{t=0}^{\infty}\gamma^{t}r_{t}^{\tau}\right],\\
        &\mathrm{s.t.} \forall ~ \tau\sim\mathcal{T}, ~\mathbb{E}_{\rho_{0},\pi}\left[\sum_{t=0}^{\infty}\gamma^{t}c_{i,t}^{\tau}\right]\leq \bar{c}_{i}^{\tau}, \\
        & ~ ~\quad \forall ~ \tau\sim\mathcal{T}, (s_{t}^{\tau},a_{t}^{\tau},r_{t}^{\tau},c_{i,t}^{\tau})\in \mathcal{D}^{\tau},
        \end{aligned}
    \end{equation}
\noindent where $c_{i}^{\tau}$ is the threshold for the $i$-th cost constraint of task $\tau$, and $\mathcal{D}^{\tau}$ is the offline dataset for task $\tau$. $r_{t}^{\tau}$ and $c_{i,t}^{\tau}$ are shorthand for $r_{t}^{\tau} = r(s_{t}^{\tau},a_{t}^{\tau})$ and $c_{i,t}^{\tau} = r(s_{t}^{\tau},a_{t}^{\tau})$.
\end{myDef}

To address the multi-task offline safe RL problem defined in Definition~\ref{def-1}, we introduce contextual representation information to better distinguish and characterize each task's features. Therefore, multi-task offline safe RL is extended to contextual multi-task offline safe RL. Since contextual multi-task offline safe RL objective is consistent with those of multi-task offline safe RL, we define contextual multi-task offline safe RL as shown in Definition~\ref{def-2}.     
\begin{myDef}
    \label{def-2} 
    \textit{(CMTOS) Given a task $\tau\!\!\sim\!\!\mathcal{T}$, where $\mathcal{T}$ is a set of tasks, Contextual Multi-task Offline Safe RL~(CMTOS) is defined as a tuple $\{(\mathcal{S}^{\tau}\!,\mathcal{A}^{\tau}\!, \mathcal{P}^{\tau}\!, r^{\tau}\!, c_{i}^{\tau}\!,\rho_{0}^{\tau},\gamma,\mathcal{M}(\tau))\}_{\tau}^{\mathcal{T}}$ with the objective written as Eq.~(\ref{equ3_3}). Where $\mathcal{M}$ is a function that maps the contextual information of task $\tau$ to the MDP parameters.}
\end{myDef}

\subsection{Conditional Generation and Cost Constraint Strategy}
In response to the constrained optimization problem defined in Definition~\ref{def-1} and~\ref{def-2}, employing traditional temporal-difference methods directly for policy evaluation may lead to significant extrapolation errors due to the influence of OOD actions in the offline dataset. Furthermore, in multi-task scenarios, temporal-difference methods require evaluating Q-values for rewards and costs across multiple tasks, further exacerbating the impact of extrapolation errors on the policy. To capture the multimode distribution of safe trajectories across multi-tasks, we transform the multi-task safe policy optimization problem defined in Definition~\ref{def-1} and~\ref{def-2} into a conditional generative model using the diffusion model.

\begin{lemma}
\label{lem: objective}
The objective of multi-task offline safe RL is defined through a conditional generation model as follows:
\begin{equation} 
    \begin{aligned}
        \label{equ_13}       
        \max_{\theta}\mathbb{E}_{\tau_{c}\sim (U_{\tau\sim\mathcal{T}}D_{\tau})}\log p_{\theta}\big[x_{0}(\tau_{c})|y(\tau_{c},\tau)\big],
    \end{aligned}
\end{equation}
\noindent where $x_{0}(\tau_{c})$ is the expected sequence under trajectory $\tau_{c}$, and $y(\tau_{c},\tau)$ is the generation condition task $\tau$ and trajectory $\tau_{c}$. The conditional generation objective aims to estimate the conditional distribution $p_{\theta}$ to generate the expected sequence $x_{0}(\tau_{c})$ from the condition $y(\tau_{c},\tau)$.
\end{lemma}

In the conditional generation model, $x_{0}(\tau_{c})$ is defined as a trajectory sequence of actions or action-state pairs, and $y(\tau_{c},\tau)$ comprises the observation sequence, reward, cost, and task contextual representation. Considering the continuity of action sequences in tasks such as autonomous driving and robotic manipulation, we employ action sequences as the generated sequence $x_{0}(\tau_{c})$ in the conditional generation model. On the other hand, to ensure that the generated action sequences meet the requirements of multi-task safe RL, we incorporate the observation state sequence, task contextual information, costs, and rewards into the conditional sequence $y(\tau_{c},\tau)$. Therefore, the generated and conditional sequence in the conditional generation model are defined as follows:
\begin{equation} 
    \begin{aligned}
        \label{equ_14}
       &x_{k}({\tau_{c}}) = (a_t, a_{t+1},\cdots, a_{t+h-1})_{k},    
    \end{aligned}
\end{equation}
\begin{equation} 
    \begin{aligned}
        \label{equ_15}
        &y(\tau_{c},\tau) =
        \begin{bmatrix}
         \mathcal{O}(\tau_{c}), \mathcal{M}(\tau), \mathcal{R}(\tau_{c}), \mathcal{C}(\tau_{c})
       \end{bmatrix}
       \end{aligned},
\end{equation}
\noindent where $k$ denotes the step in the forward process, $t$ represents the time step at which an action is visited in trajectory $\tau_{c}$, and $x_{k}(\tau_{c})$ is a noisy action sequence of length $h$. $\mathcal{R}(\tau_{c})$ and $\mathcal{C}(\tau_{c})$ are the reward and cost conditions. $\mathcal{O}(\tau_{c})$ is the historical state sequence containing $l$ steps, denoted as $\mathcal{O}(\tau_{c})= (s_{t-l+1},\cdots, s_t)$.  $\mathcal{M}(\tau)$ is the contextual information of the task $\tau$, which is detailed in Section~\ref{context_rep}.


According to Lemma~\ref{lem: objective}, multi-task conditional generation aims to learn a conditional distribution that produces actions maximizing rewards while satisfying cost constraints. We train the reverse diffusion process $p_{\theta}$ using supervised learning through the parameterized noise model $\epsilon_{\theta}$. Additionally, we employ classifier-free guidance to train the diffusion model and follow the DDPM~\cite{NEURIPS2020_4c5bcfec} method for denoising. Therefore, the loss function for the diffusion model is defined as follows:
\begin{equation} 
    \begin{aligned}
        \label{equ_16}
        &\!\!\!\!\!\!\mathcal{L}(\theta)=\mathbb{E}_{k,\tau_{c},\tau}\Big[\big{\Vert}\epsilon-\epsilon_{\theta}\big(x_{k}(\tau_{c}), y'(\tau_{c},\tau), \\
        &\quad \quad ~ \beta(\mathcal{R}(\tau_{c}),\mathcal{C}(\tau_{c}))+(1-\beta)\varnothing, k \big)\big{\Vert}^{2}\Big],\\  
    \end{aligned}
\end{equation}
\noindent where $\beta$ is sampled from a Bernoulli distribution with probability $p_{_\beta}$. $k$ is the denoising step $k\!\!\in\!\!\{1,2,\cdots,K\}$, and $K$ is the maximum number of denoising steps. First, we sample the trajectory $\tau_{c}$ for each task $\tau$, then sample the noise $\epsilon\!\!\sim\!\! \mathcal{N}(0,I)$ and the time step $k$. Subsequently, we construct the state sequence and task contextual information $y'(\tau_c,\tau)\!\!=\!\!(\mathcal{O}(\tau_{c}),\mathcal{M}(\tau))$, reward $\mathcal{R}(\tau_c)$ and cost $\mathcal{C}(\tau_c)$ conditions, and a noisy action sequence $x_{k}(\tau_c)$. Finally, we predict the noise as $\hat{\epsilon}_{\theta}$, ignoring the reward and cost conditions with probability $p_{_\beta}$. 


During the denoising process, we start with Gaussian noise $x_{k}(\tau_c)$ and iteratively sample and denoise it to produce the optimal action sequence $x_{0}(\tau_c)$ that meets the cost constraints. To ensure the generation of high-likelihood sequences, we employ low-temperature sampling techniques for $x_{k-1}(\tau_{c})=(\mu_{k-1},\alpha\sum_{k-1})$ with the parameter $\alpha\!\!\in\!\![0,1)$. During sampling, the mean $\mu_{k\!-\!1}$ and variance $\sum_{k\!-\!1}$ are computed through reverse denoising based on the action sequence $x_{k}(\tau_c)$ and perturbation noise $\hat{\epsilon}$. To guide the generation of noise that satisfies the conditions during the denoising process, we take the state sequence $\mathcal{O}(\tau_{c})$ and contextual information $\mathcal{M}(\tau)$ as necessary conditions and utilize classifier-free guidance to constrain costs and maximize reward. Consequently, the perturbation noise in the intermediate steps from $x_{k}(\tau)$ to $x_{0}(\tau)$  during the denoising is defined as stated in Proposition~\ref{pro:class-free}.
\begin{proposition}
\label{pro:class-free}
The perturbation noise $\hat{\epsilon}$ in the intermediate steps of the denoising process, which takes the state sequence $\mathcal{O}(\tau_{c})$ and task contextual information $\mathcal{M}(\tau)$ as necessary conditions and employs classifier-free guidance to constrain cost and maximize reward, is defined as:
\vspace{-0.1cm}
\begin{equation} 
    \begin{aligned}
        \label{equ_03}
        &\hat{\epsilon}=\kappa\Big[\epsilon_{\theta}\big(x_{k}(\tau_{c}),y'(\tau_{c},\tau),\mathcal{R}(\tau_{c}),k\big) \\
        & \quad +\epsilon_{\theta}\big(x_{k}(\tau_{c}),y'(\tau_{c},\tau),\mathcal{C}(\tau_{c}),k\big)\Big]\\
        &\quad +(1-2\kappa)\epsilon_{\theta}\big(x_{k}(\tau_{c}),y'(\tau_{c},\tau),\varnothing,k\big),
    \end{aligned}
\end{equation}
\noindent where, $y'(\tau_{c},\tau)$ represents the state sequence and task contextual information $y'(\tau_{c},\tau)= (\mathcal{O}(\tau_{c}),\mathcal{M}(\tau))$. $\mathcal{R}$ and $\mathcal{C}$ are the reward and cost conditions, $\kappa$ is the classifier-free guidance coefficient, and $\epsilon_{\theta}$ denotes the parameterized noise distribution with parameters $\theta$.
\end{proposition} 

The proofs and discussions of  Proposition~\ref{pro:class-free} are in \textbf{Appendix~\ref{sup: class-free}}. To guide the conditional generation model in generating action sequences with the desired reward and costs, we employ the reward $\mathcal{R}(\tau_{c})=\sum_{t=j}^{n}\gamma^{t-j}r_{t}^{\tau}$ and cost $\mathcal{C}(\tau_{c})=\sum_{t=j}^{n}\gamma^{t-j}c_{i,t}^{\tau}$ of the remaining trajectories as conditions for classifier-free guidance. Note that due to the varying scales of rewards and costs across different tasks, we normalize the rewards $\mathcal{R}(\tau_{c})$ and cost $\mathcal{C}(\tau_{c})$ conditions for each task to fall within the range $[0,1]$. 
From the conditional generation model in Eq.~(\ref{equ_13}) and the perturbation noise in the denoising process in Eq.~(\ref{equ_03}), it is evident that our cost constraint strategy, based on the conditional generation model, constrains the costs of each task through the remaining cost $\mathcal{R}(\tau_c)$ of the trajectory. Therefore, after training the model, we only need to adjust the cost conditions $\mathcal{R}(\tau_c)$ to achieve different cost constraints without further training. This flexible cost constraint method of the conditional generation model allows us to derive the conclusions presented in Remark~\ref{rem: cost} and Remark~\ref{rem:state-wise cost}.


\begin{remark}
\label{rem: cost}
After completing model training, cost conditions $\mathcal{C}(\tau_c)$ can be constructed based on the accumulated cost of the current trajectory and the cost thresholds $\bar{c}_{i}^{\tau}$ to implement safety constraints at various cost thresholds without requiring further training. Additionally, each task can be assigned a different cost threshold $\bar{c}_{i}^{\tau}$.
\end{remark}

\begin{remark}
\label{rem:state-wise cost}
The CDCP employs the remaining cost $\mathcal{C}(\tau_{c})=\sum_{t=j}^{n}\gamma^{t-j}c_{i,t}^{\tau}$ of the trajectory as the cost constraint condition $\mathcal{C}(\tau_{c})$. Therefore, setting the cost condition to zero can convert the accumulated cost into a state-wise cost constraint.
\end{remark}
\begin{algorithm}[htb!] \small
    \caption{CDCP}
    \label{alg:our_algorithm_train}
    \tcp{Training Process}
    \KwIn{ Multi-task expert trajectory prompts $\mathcal{T}^{p}$; pre-trained language model $L_{\phi}$;
    training task $\mathcal{T}$; dataset $ \mathcal{D}_{\tau\sim\mathcal{T}} =\{(s,a,r,c,s',d)_{i=0}^{n}\}_{\tau\sim\mathcal{T}}$;   }
    \KwOut{ Diffusion network parameters $\epsilon_{\theta} $;
    Weight layer parameters $\mathcal{W}$;
    }
    \For{ $n=1$ \textbf{to} $\mathcal{N}$}{
        \For {each task}{
            Sample action sequence $x_{0}(\tau_{c})$ of length $h$ and corresponding state history $\mathcal{O}(\tau_c)$ of length $l$ from $\mathcal{D}_{\tau}$ with batch size $m$\;
            Compute normalized reward $\mathcal{R}(\tau_c)$ and cost $\mathcal{C}(\tau_c)$ conditions for the sequence under $\tau_{c}$\;
            Sample a trajectory prompt $\tau^{p}$ of length $n$ from $\mathcal{T}^{p}$\; 
            Get the textual information $\tau^{t}$ of the task $\tau$\;  
            Get the task feature vector $\tau_{z}^{t}=L_{\phi}(\tau^{t})$ \;
        }
        Get a batch of sequence $\mathcal{B}=\{x_{0}(\tau_{c}),\mathcal{O}(\tau_c), [\tau^{p},\tau^{t}_{z}], \mathcal{R}(\tau_{c}),\mathcal{C}(\tau_{c})\}_{\tau\sim\mathcal{T}}$\;
        Randomly sample a diffusion timestep $k$ and obtain noisy sequences $x_{k}(\tau_{c})$\;
        Omit the reward $\mathcal{R}(\tau_c)$ and cost $\mathcal{C}(\tau_c)$ condition with probability $\beta\sim \text{Bern}(p_{\beta})$\;
        Obtain the task's feature vector by encoding the textual information $\tau^{t}$ with the pre-trained language model $L_{\phi}$\;
        Compute the loss $\mathcal{L}(\mathcal{W})$ through the Eq.~(\ref{equ_11})\;
        Compute the loss $\mathcal{L}(\theta)$ through the Eq.~(\ref{equ_16}) and~(\ref{equ_45})\;
        Update the weight coefficient $\mathcal{W}(t+1) = \mathcal{W}(t) +\lambda \nabla\mathcal{L}{(\mathcal{W})}$\;
        Update the diffusion model $\theta(t+1) = \theta(t) +\lambda \nabla\mathcal{L}(\theta)$
    }    
\end{algorithm}
\vspace{-0.25cm}
\subsection{Contextual Prompts for Multi-task}
\label{context_rep} 
From Lemma~\ref{lem: objective} and Eq.~(\ref{equ_14}) and ~(\ref{equ_15}), it is evident that the goal of our multi-task conditional generation model is to generate an optimal action sequence $x_{0}(\tau_c)$ that satisfies cost constraints, given the observed state sequence $\mathcal{O}(\tau_c)$, task contextual information $\mathcal{M}(\tau)$, cost constraints $\mathcal{C}(\tau_c)$, and reward maximization $\mathcal{R}(\tau_c)$ conditions. The similarity in observed states across tasks with different actions can result in confusion or conflict in the mapping relationship between observed states and actions. To address this problem, we introduce additional contextual information to differentiate the observations among various tasks. Different tasks exhibit significant differences due to their inherent characteristics. For instance, in autonomous driving, the action sequences of vehicles differ significantly between highway and urban environments, as well as between sparse and dense traffic conditions. Similarly, in robotics manipulation, action sequences for tasks such as pushing, picking, and grasping show notable differences. Therefore, we utilize textual descriptions of tasks to provide characteristic information about them. Additionally, while textual descriptions of tasks can provide information about the category and characteristics of the tasks, they may not effectively convey the mapping relationships between actions and states within the tasks. Therefore, we also employ the state action sequences of the tasks as contextual prompts to provide this information.

Building on the above analysis, we propose a contextual prompting method to address the characterization of multi-task features in the conditional generation model outlined in Proposition~\ref{pro:class-free}. This method utilizes task textual descriptions to represent the categories and characteristics of the tasks and employs trajectory information to elucidate the mapping relationships between states and actions. Therefore, we define the contextual information $\mathcal{M}(\tau)$ as follow:
\begin{equation} 
    \begin{aligned}
    \label{equ_34}
    \mathcal{M}(\tau) = \big[\tau^{p},\tau^{t}_{z}\big],    
    \end{aligned}
\end{equation}
\vspace{-0.4cm}
\begin{equation} 
    \begin{aligned}
    \label{equ_35}
       &\tau^{p} =
       \begin{bmatrix}
       s_{i}^{*} & a_{i}^{*} & s_{i+1}^{*} & \cdots & a_{i+\chi-1}^{*}
       \end{bmatrix},
    \end{aligned}
\end{equation}
\noindent where $\tau^{p}$ represents the task's trajectory of action state pairs. $s_{i}^{*}$ and $a_{i}^{*}$ are the state action pairs from the expert trajectory, and $\chi$ is the length of the state action pairs in the prompt trajectory, where $\chi$ is shorter than the full trajectory length. $\tau_{z}^{t}$ is the hidden feature vector of the task's textual description encoded by the pre-trained language model $\tau_{z}^{t}=L_{\phi}(\tau^{t})$. $L_{\phi}$ is the pre-trained language model, and $\tau^{t}$ is the task's textual description. We utilize the encoded task feature vector $\tau_{z}^{t}$ to represent the task's category and characteristics. The contextual information $\mathcal{M}(\tau)$ conveys the task's category and characteristics through textual information and establishes the mapping relationships between actions and states across tasks using expert trajectory prompts. Consequently, this contextual prompting method enhances the CDCP algorithm's ability to adapt to unseen tasks under few-shot samples. Therefore, we derive the conclusion in Remark~\ref{rem:zero-shot}.

\begin{remark}
\label{rem:zero-shot}
The contextual prompting method combines task textual descriptions with expert trajectory prompts, improving the CDCP algorithm's adaptability to unseen tasks.
\end{remark}
\vspace{-0.2cm}

\begin{algorithm}[htb] \small
    \caption{CDCP}
    \label{alg:our_algorithm_eval}    
    \tcp{Evaluation Process}    
    \KwIn{Diffusion network parameters $\epsilon_{\theta} $;
    weight layer parameters $\mathcal{W}$; few-shot prompts $\mathcal{T}^{p}$; 
    pre-trained language model $L_{\phi}$;}
    Given a task $\tau$, get the textual information $\tau^{t}$ \;
    Reset the environment and obtain the initial state history $\mathcal{O}(\tau_c)$\;
    Set the desired reward $\mathcal{\bar{R}}(\tau_c)$ and cost $\mathcal{\bar{C}}(\tau_c)$\;
    Set low-temperature sampling scale $\alpha$, and classifier-free guidance scale $\kappa$\;
    \For{$t=0$ \textbf{to} $t_{\max}$}{
        Initialize $x_{k}(\tau_c)\sim \mathcal{N}(0,\alpha I)$\;
        Sample a trajectory prompt $\tau^{p}\sim \mathcal{T}^{p}$\;
        Get the task feature vector $\tau_{z}^{t}=L_{\phi}(\tau^{t})$ from the task textual information $\tau^{t}$\;
        Formulate the condition $[\mathcal{O}(\tau_c),\tau^{p},\tau^{t}_{z},\bar{\mathcal{R}}(\tau_c),\bar{\mathcal{C}}(\tau_c)]$\;
        \For{$k=K$ \textbf{to} $1$ }{
        Compute the perturbed noise $\hat{\epsilon}$ through the Eq.~(\ref{equ_03})\;
        $(\mu_{k-1},\sum_{k-1})\leftarrow$ Denoise $(x_{k}(\tau),\hat{\epsilon})$\;
        $x_{k-1}(\tau_c)\sim\mathcal{N}(\mu_{k-1},\alpha\sum_{k-1})$
        }   
        Get the first action from $x_{0}(\tau_c)$ as the current action to interact with the environment \;
        Obtain the next state and update the $\mathcal{O}(\tau_c)$\;
    } 
\end{algorithm}

\vspace{-0.4cm}
\subsection{Gradient Loss Synchronization Learning}
Due to the varying complexity of mapping relationships across different tasks, the loss values for each task differ significantly. This causes some tasks to dominate and interfere with others during simultaneous training. Additionally, these differences in mapping relationships result in varying convergence rates, leading to instability in the multi-task training process. The varying magnitudes of losses and convergence rates across tasks result in gradient interference during simultaneous training. Inspired by GradNorm~\cite{chen2018gradnorm}, we introduce a gradient loss synchronization strategy that accounts for the magnitude of each task's loss and the rate of change in the loss to address the gradient interference issue. 

Based on the above analysis, we assign weights to each task to eliminate gradient interference between tasks when computing the overall multi-task loss. These weights are determined by the magnitude of each task's loss and its rate of change, replacing the previous approach of directly averaging the losses of all tasks.
\begin{equation} 
    \begin{aligned}
        \label{equ_45}
        \mathcal{L}(\theta) = \sum_{\tau\sim\mathcal{T}}\omega_{\tau}*\mathcal{L}_{\tau}(\theta),
    \end{aligned}
\end{equation}
\noindent where $\omega_{\tau}\!\!\in\!\! \mathcal{W}$, $\mathcal{W}$ is the weight layer parameters of the network, and $\omega_{\tau}$ is the weight layer parameter for task $\tau$. $\mathcal{L}_{\tau}$ is the loss gradient for task $\tau$, and $\mathcal{L}$ is the overall loss gradient. Our gradient loss synchronization strategy aims to automatically adjust the weight layer parameters $\mathcal{W}$ based on each task's difficulty level and convergence rate, ensuring that the gradient loss for each task decays stably and synchronously. We introduce the gradient loss norm and the gradient Loss decay rate for each task to design the adjustment strategy for the weight layer parameters $\mathcal{W}$. The gradient loss norm is defined as follows:
\vspace{-0.1cm}
\begin{equation} 
    \begin{aligned}
        \label{equ_07}
         \mathcal{G}_{\mathcal{W}}^{\tau}= \big{\Vert}\nabla_{\mathcal{W}}\omega_{\tau}\mathcal{L}_{\tau} \big{\Vert}_{2} ,
    \end{aligned}
\end{equation}
\vspace{-0.2cm}
\begin{equation} 
    \begin{aligned}
        \label{equ_08}
        \overline{\mathcal{G}}_{\mathcal{W}} = \mathbb{E}_{\tau\sim\mathcal{T}} \big[\mathcal{G}_{\mathcal{W}}^{\tau}\big],
    \end{aligned}
\end{equation}
\noindent where $\mathcal{G}_{\mathcal{W}}^{\tau}$ is the L2 norm of the gradient loss of the weighted loss of a single task concerning the weight $\mathcal{W}$. $\overline{\mathcal{G}}_{\mathcal{W} }$ is the average gradient loss norm across all tasks. We measure each task's gradient loss decay rate using the ratio of gradient losses at different stages. Therefore, the gradient loss decay rate for each task is defined as follows:
\begin{equation} 
    \begin{aligned}
        \label{equ_09}
        \Delta \mathcal{L}_{\tau}(t) = \lambda_{\mathcal{L}}\frac{\mathcal{L}_{\tau}(t)}{\mathcal{L}_{\tau}(0)}+(1-\lambda_{\mathcal{L}})\frac{\mathcal{L}_{\tau}(t)}{\mathcal{L}_{\tau}(t-1)}, 
    \end{aligned}
\end{equation}
\vspace{-0.2cm}
\begin{equation} 
    \begin{aligned}
        \label{equ_10}
        \overline{\Delta \mathcal{L}}_{\tau}(t) = \frac{\Delta \mathcal{L}_{\tau}(t)}{\mathbb{E}_{\tau\sim\mathcal{T}} \big[\Delta \mathcal{L}_{\tau}(t)\big]},
    \end{aligned}
\end{equation}
\noindent where $\mathcal{L}_{\tau}(0)$, $\mathcal{L}_{\tau}(t-1)$, and $\mathcal{L}_{\tau}(t)$ represent the gradient losses of task $\tau$ at steps $0$, $t-1$ and $t$, respectively. $\Delta\mathcal{L}_{\tau}(t)$ represents the measure of the gradient loss decay rate of task $\tau$. $\overline{\Delta \mathcal{L}}_{\tau}(t)$ is the normalized measure of the gradient loss decay rate of the task $\tau$. Note that $\Delta\mathcal{L}_{\tau}(t)$ and $\overline{\Delta \mathcal{L}}_{\tau}(t)$ are inversely related to the measure of training rate. Using the gradient loss norm shown in Eq.~(\ref{equ_08}) and the gradient loss decay rate shown in Eq.~(\ref{equ_10}), we establish the target gradient loss for each task. Consequently, the loss function for the weight layer parameters $\mathcal{W}$ is defined as follow:
\vspace{-0.05cm}
\begin{equation} 
    \begin{aligned}
        \label{equ_11}
        \mathcal{L}(\mathcal{W}) = \sum_{\tau\sim\mathcal{T}} \bigg{\vert} \mathcal{G}_{\mathcal{W}}^{(\tau)}(t) - \overline{\mathcal{G}}_{\mathcal{W}}(t) \times \Big[\overline{\Delta \mathcal{L}}_{\tau}(t)\Big]^{\eta}\bigg{\vert}_{1}, 
    \end{aligned}
\end{equation}
\vspace{-0.05cm}
\noindent where $\eta$ is the parameter for adjusting the gradient loss decay rate of each task, and $\overline{\mathcal{G}}_{\mathcal{W}}(t) \times \big[\overline{\Delta \mathcal{L}}_{\tau}(t)\big]^{\eta}$ is the target gradient loss norm for each task, which we treat as a constant.
\vspace{-0.2cm}
\subsection{Practical Algorithm}
To facilitate understanding of the CDCP algorithm's implementation, we provide a detailed explanation of a practical instance. The pseudocode for training and testing the CDCP algorithm is presented in Algorithm~\ref{alg:our_algorithm_train} and~\ref{alg:our_algorithm_eval}, respectively. The prompt trajectories $\mathcal{T}^{p}$ utilized during the CDCP algorithm's training and testing phases are expert trajectories that satisfy safety constraints. Additionally, the task's textual representation $\tau^{t}$ directly incorporates descriptions of the driving scenarios and traffic density, encoded into task feature vectors $\tau_{z}^{t}$ using the pre-trained language model $L_{\phi}$. The pre-trained model's parameters $\phi$ remain unchanged throughout the training and testing processes. 

In the training phase, multiple tasks require synchronized training under unified reward and cost conditions. Given the differences in rewards and costs across tasks, we normalize the cumulative rewards and costs of the remaining trajectories to a unified scale. The normalized cumulative rewards and costs are defined as follows:
\begin{figure*}[htp!]
\centering
\begin{subfigure}[b]{0.325\textwidth}
  \includegraphics[width=\textwidth]{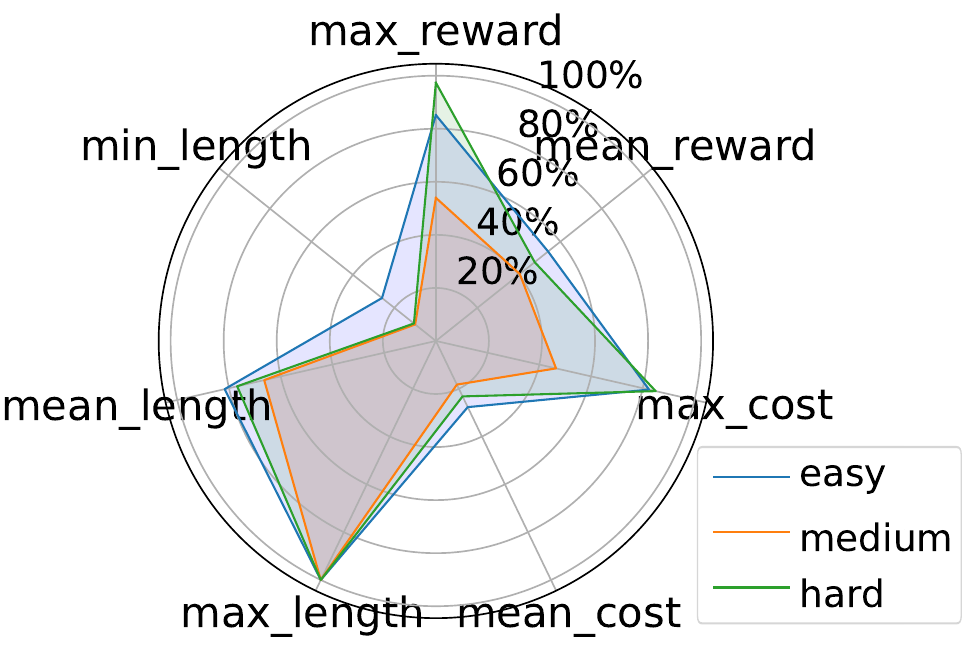}
  \caption{Sparse}
  \label{fig:sub1}
\end{subfigure}
\hspace{-0.25cm}
\begin{subfigure}[b]{0.325\textwidth}
  \includegraphics[width=\textwidth]{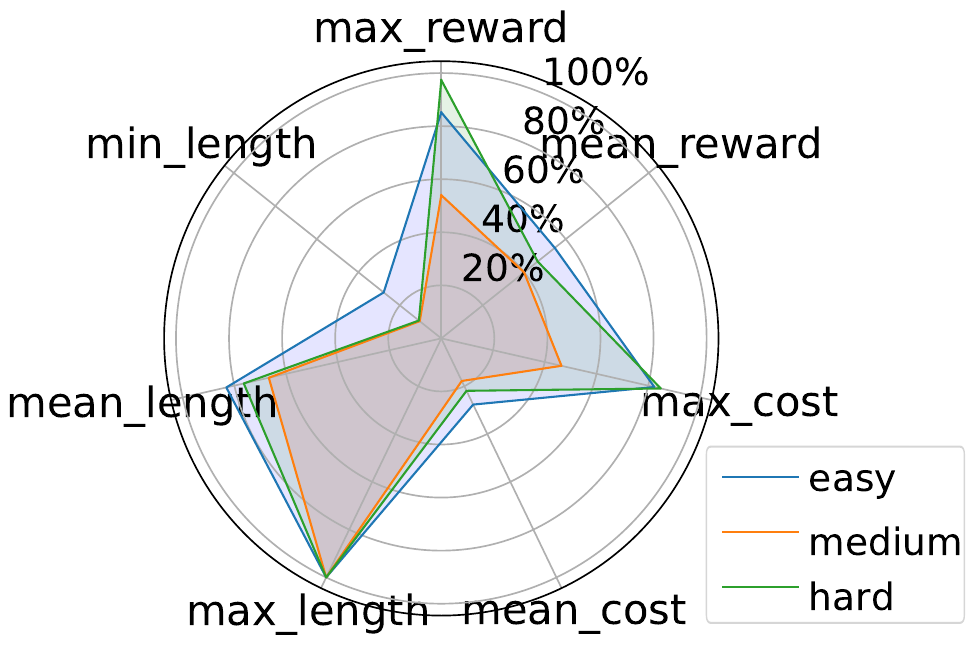}
  \caption{Mean}
  \label{fig:sub2}
\end{subfigure}
\hspace{-0.25cm}
\begin{subfigure}[b]{0.325\textwidth}
  \includegraphics[width=\textwidth]{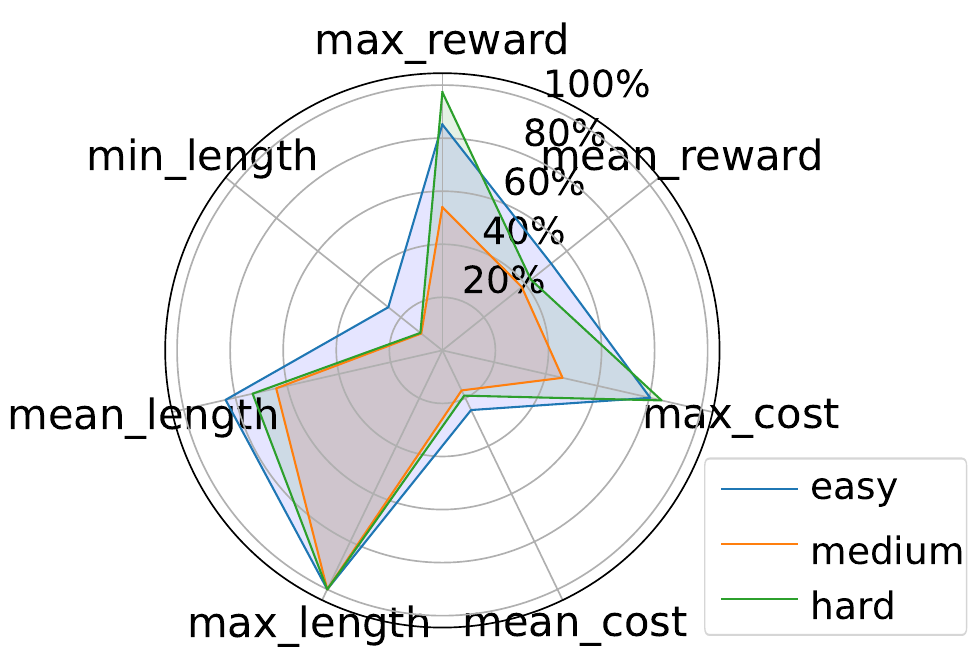}
  \caption{Dense}
  \label{fig:sub3}
\end{subfigure}
\caption{The radar charts depicting the characteristics of sample data from nine different driving tasks with varying traffic scenarios or densities. The charts display statistical results of the sample trajectories' rewards, costs, and trajectory lengths for different tasks. (a) Tasks with sparse traffic density across different traffic scenarios; (b) Tasks with mean traffic density across different traffic scenarios; (c) Tasks with dense traffic density across different traffic scenarios.}
\label{fig:data_episode}
\end{figure*}

\begin{figure*}[htp!]
  \centering
  \begin{subfigure}[b]{0.3\textwidth}
    \includegraphics[width=\textwidth]{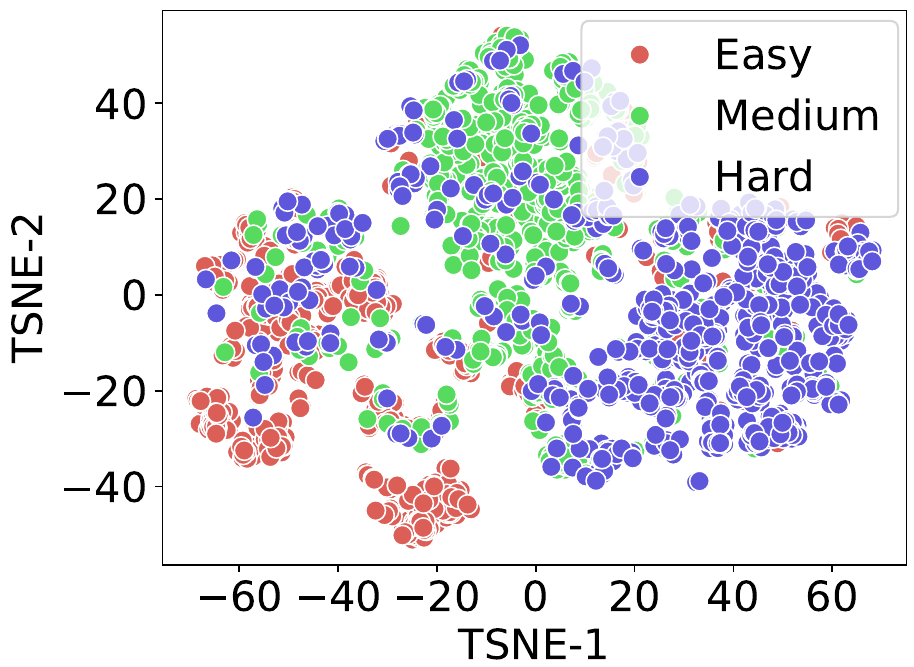}
    \caption{Sparse}
  \end{subfigure}%
  \hspace{0.5cm}
  \begin{subfigure}[b]{0.3\textwidth}
    \includegraphics[width=\textwidth]{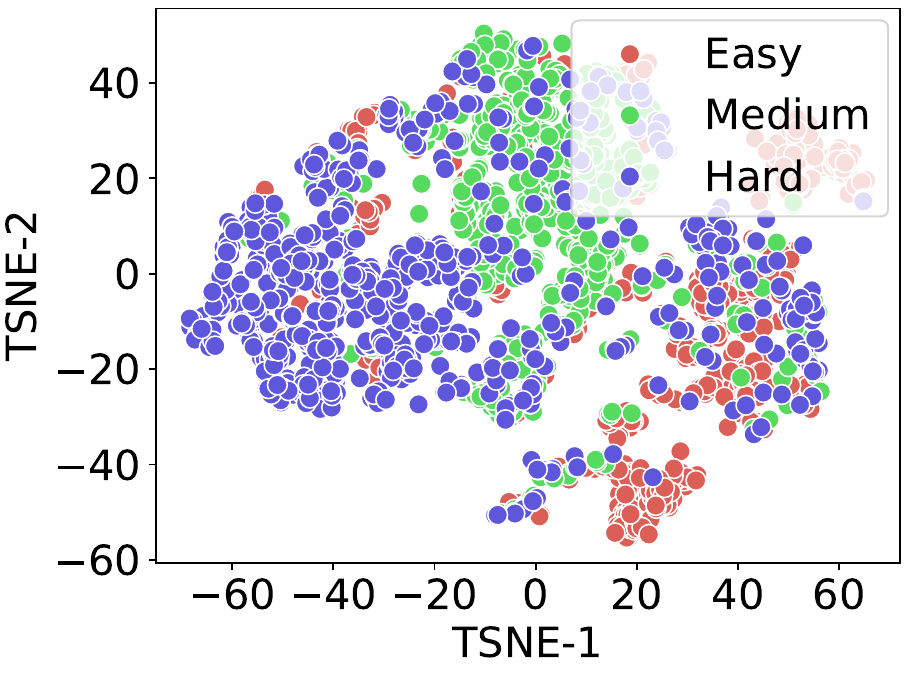}
    \caption{Mean}
  \end{subfigure}%
  \hspace{0.5cm}
  \begin{subfigure}[b]{0.3\textwidth}
    \includegraphics[width=\textwidth]{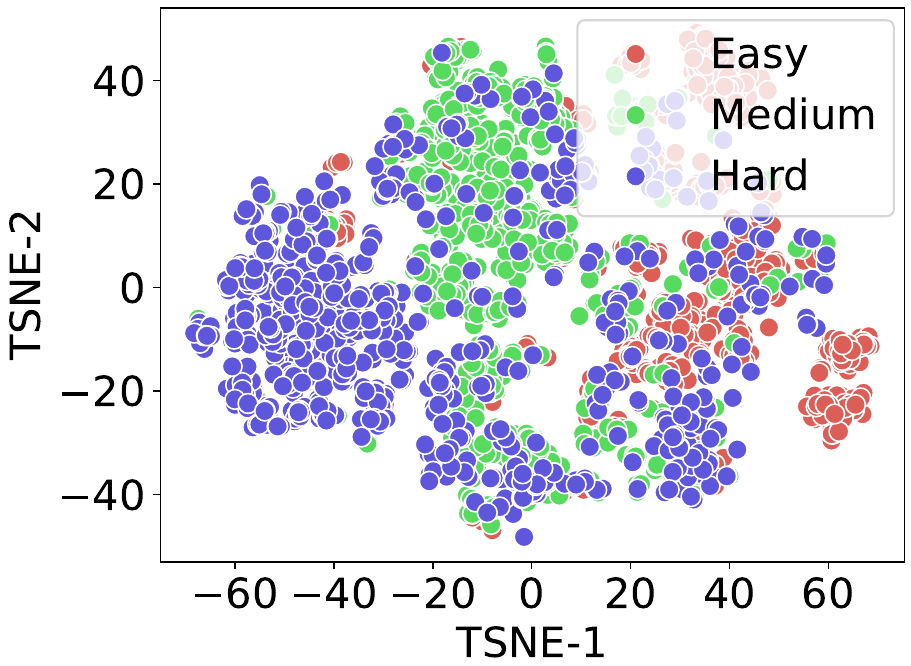}
    \caption{Dense}
  \end{subfigure}
  \caption{The characteristic distribution of sample action-state pairs for nine driving tasks with varying traffic scenarios or densities. For each task, 10,000 action-state pairs are randomly sampled and visualized using T-SNE for dimensionality reduction.  (a) Tasks with sparse traffic density across different traffic scenarios; (b) tasks with mean traffic density across different traffic scenarios; (c) tasks with dense traffic density across different traffic scenarios.}
  \vspace{-0.2cm}
  \label{fig:data_step}
\end{figure*}
\begin{equation} 
    \begin{aligned}
        \label{equ_19}
         \mathcal{R}(\tau_c) = \frac{t_m}{(t_m-t_c)\mathcal{R}_{m}(\tau)}\sum_{t=t_c}^{t_m} \gamma^ {t}r_{t}^{\tau},
    \end{aligned}
\end{equation}
\noindent where $R_{m}(\tau)$ is the maximum cumulative reward of the current task $\tau$. $t_{m}$ is the maximum step of the current trajectory $\tau_c$, and $t_c$ is the current step of the current trajectory $\tau_c$. Note that to ensure the reward value converges to the same fixed point, the cumulative reward is iteratively computed using the discounted cumulative reward. The normalization method for cumulative cost is the same as for reward in Eq.~(\ref{equ_19}).

In the testing phase, following the normalization method for reward and cost conditions used during training, we only need to set the reward condition to 1 and the cost condition to the ratio of the desired cost to the maximum task cost $\tau$. However, given that the cumulative reward of a trajectory satisfying the safety constraints is inevitably less than the task's maximum cumulative reward, we multiply the normalized reward by a correction factor $\eta_{_{\mathcal{R}}}\in(0,1)$. Consequently, our reward condition is set to $\bar{\mathcal{R}}(\tau_c)=\eta_{_{\mathcal{R}}}$. Additionally, during the training phase, the maximum cumulative cost employed for normalization is the non-discounted cumulative cost, whereas the remaining cost employs the discounted cost. This results in the normalized cost $\mathcal{C}(\tau_c)$ being less than 1. Therefore, we multiply the expected cumulative cost by a correction factor $\eta_{_{\mathcal{C}}}$. The cost condition is defined as follows:
\begin{equation} 
    \begin{aligned}
        \label{equ_20}
     \bar{\mathcal{C}}(\tau_c) = \eta_{_\mathcal{C}}\frac{\bar{c}_{i}^{\tau}}{\mathcal{C}_{\max}(\tau)}(\bar{c}_{i}^{\tau}-\mathcal{C}_{p}),
    \end{aligned}
\end{equation}

\noindent where $\bar{c}_{i}^{\tau}$ is the cos threshold of task $\tau$, $\mathcal{C}_{p}$ represents the cumulative cost incurred up to the current step of the trajectory, and $\mathcal{C}_{\max}(\tau)$ is the maximum cumulative cost for task $\tau$. The correction factor $\eta_{_{\mathcal{C}}}\in(0,1)$.

%% file: sec/5_experimental.tex
\begin{table*}[htp!]
    \centering
    \caption{The rewards and costs of the CDCP and the baseline algorithms improved from offline safe RL. The rewards and costs for each of the 9 tasks during simultaneous training and the average reward and cost across the nine tasks are recorded. The table shows results obtained by training with 3 random seeds, recording the best model for each seed, and testing each best model across 10 episodes to gather statistical data. Note that the average reward and cost in the last row of the table are calculated by first averaging the reward and cost for all tasks under the same random seed and episode, then computing the mean values of reward and cost across different random seeds and episodes. The mean and standard deviation recorded in the last row of the table represent the results of this final step for each random seed and episode.}
    \vspace{-0.1cm}
    \label{table: offline_Safe}
    \setlength{\tabcolsep}{2.2mm}{\begin{tabular}{ccccccccc}
    \hline
    \hline \\ [-6pt]
    \multirow{2}{*}{Method} & \multicolumn{2}{c}{MTCPQ} & \multicolumn{2}{c}{MTCOptiDICE} & \multicolumn{2}{c}{MTCDT}   & \multicolumn{2}{c}{CDCP(Ours)}  \\ \cmidrule(lr{0pt}){2-3}  \cmidrule(lr{0pt}){4-5}  \cmidrule(lr{0pt}){6-7} \cmidrule(lr{0pt}){8-9}
                            & Reward$\uparrow$            & Cost$\downarrow$          & Reward$\uparrow$               & Cost$\downarrow$             & Reward$\uparrow$             & Cost$\downarrow$           & Reward$\uparrow$            & Cost$\downarrow$          \\ \hline \\ [-6pt]
    EasySparse              & -2.97$\pm$1.02    & 0.33$\pm$0.47  & 242.81$\pm$85.52    & 57.44$\pm$21.44   & 175.75$\pm$112.97 & 36.97$\pm$26.53 & 326.06$\pm$52.41 & 25.21$\pm$5.17 \\
    EasyMean                & -2.97$\pm$1.02   & 0.33$\pm$0.47  & 233.32$\pm$66.12    & 54.33$\pm$16.44   & 162.32$\pm$92.57  & 33.18$\pm$20.63 & 319.28$\pm$48.60 & 24.68$\pm$5.22 \\
    EasyDense               & -2.97$\pm$1.02   & 0.33$\pm$0.47  & 208.75$\pm$51.40    & 47.19$\pm$12.28   & 155.48$\pm$79.75  & 30.82$\pm$16.97 & 316.47$\pm$41.03 & 24.39$\pm$6.41 \\
    MediumSparse            & -4.31$\pm$1.06   & 1.21$\pm$0.29  & 177.55$\pm$78.63    & 38.44$\pm$17.18   & 70.04$\pm$35.42   & 14.60$\pm$7.83  & 242.75$\pm$26.06 & 20.68$\pm$6.19 \\
    MediumMean              & -4.31$\pm$1.06   & 1.21$\pm$0.29  & 173.61$\pm$67.32    & 37.42$\pm$14.33   & 72.95$\pm$41.91   & 15.31$\pm$9.14  & 236.57$\pm$28.76 & 19.65$\pm$7.78 \\
    MediumDense             & -4.31$\pm$1.06   & 1.21$\pm$0.29  & 161.13$\pm$65.48    & 34.65$\pm$14.03   & 69.95$\pm$37.10   & 14.68$\pm$7.91  & 237.75$\pm$26.97 & 18.75$\pm$6.31 \\
    HardSparse              & -2.56$\pm$1.30   & 0.33$\pm$0.47  & 147.62$\pm$79.43    & 32.66$\pm$19.05   & 94.26$\pm$49.71   & 19.57$\pm$9.94  & 257.98$\pm$51.74 & 24.94$\pm$8.34 \\
    HardMean                & -2.56$\pm$1.30   & 0.33$\pm$0.47  & 154.22$\pm$73.08    & 33.76$\pm$17.22   & 87.88$\pm$40.73   & 18.14$\pm$7.76  & 253.58$\pm$42.94 & 24.78$\pm$7.75 \\
    HardDense               & -2.56$\pm$1.30   & 0.33$\pm$0.47  & 106.05$\pm$69.57    & 22.18$\pm$16.78   & 86.34$\pm$39.00   & 16.88$\pm$7.43  & 243.63$\pm$39.60 & 24.43$\pm$6.23 \\
   \cellcolor{lightgray}{Mean}                    & \cellcolor{lightgray}{-2.62$\pm$0.42}   & \cellcolor{lightgray}{0.62$\pm$0.41}  & \cellcolor{lightgray}{178.34$\pm$34.72}    & \cellcolor{lightgray}{39.79$\pm$7.51}    & \cellcolor{lightgray}{108.33$\pm$47.14}  & \cellcolor{lightgray}{22.24$\pm$9.69}  & \cellcolor{lightgray}{\textbf{270.45$\pm$15.78}} & \cellcolor{lightgray}{\textbf{23.06$\pm$3.51}} \\ \hline
    \end{tabular}
    }
\end{table*}

\begin{table*}[htp!]
    \centering
    \caption{The rewards and costs of the CDCP and baseline algorithms improved from multi-task offline RL or offline RL. The rewards and costs for each of the 9 tasks during simultaneous training and the average reward and cost across the nine tasks are recorded. The table shows results obtained by training with 3 random seeds, recording the best model for each seed, and testing each best model across 10 episodes to gather statistical data. Note that the average reward and cost in the last row of the table are calculated by first averaging the reward and cost for all tasks under the same random seed and episode, then computing the mean values of reward and cost across different random seeds and episodes. The mean and standard deviation recorded in the last row of the table are the results of this final step for each random seed and episode. }
    \vspace{-0.1cm}
    \label{table: mtrl}
    \setlength{\tabcolsep}{4.8mm}{
    \begin{tabular}{ccccccc}
    \hline
    \hline \\ [-6pt]
    \multirow{2}{*}{Method} & \multicolumn{2}{c}{CMTDD} & \multicolumn{2}{c}{CMTDiff} & \multicolumn{2}{c}{CDCP(Ours)}  \\ 
    \cmidrule(lr{0pt}){2-3}  \cmidrule(lr{0pt}){4-5}  \cmidrule(lr{0pt}){6-7} 
                            & Reward$\uparrow$            & Cost$\downarrow$          & Reward$\uparrow$             & Cost$\downarrow$           & Reward$\uparrow$            & Cost$\downarrow$          \\ \hline \\ [-6pt]
    EasySparse              & 73.89$\pm$26.84  & 24.08$\pm$3.81 & 330.39$\pm$67.16  & 45.52$\pm$21.14 & 326.06$\pm$52.41 & 25.21$\pm$5.17 \\
    EasyMean                & 70.36$\pm$23.54  & 23.61$\pm$4.14 & 324.41$\pm$52.04  & 40.07$\pm$17.24 & 319.28$\pm$48.60  & 24.68$\pm$5.22 \\
    EasyDense               & 65.25$\pm$25.98  & 22.33$\pm$4.60  & 308.45$\pm$54.97  & 37.20$\pm$12.33  & 316.47$\pm$41.03 & 24.39$\pm$6.41 \\
    MediumSparse            & 58.20$\pm$22.40  & 21.27$\pm$3.87 & 214.62$\pm$22.39  & 22.44$\pm$5.71  & 242.75$\pm$26.06 & 20.68$\pm$6.19 \\
    MediumMean              & 57.14$\pm$20.49  & 20.71$\pm$3.91 & 213.39$\pm$35.61  & 22.20$\pm$5.98   & 236.57$\pm$28.76 & 19.65$\pm$7.78 \\
    MediumDense             & 53.60$\pm$19.62   & 18.86$\pm$3.11 & 205.67$\pm$26.89  & 20.55$\pm$5.22  & 237.75$\pm$26.97 & 18.75$\pm$6.31 \\
    HardSparse              & 67.35$\pm$24.05  & 21.90$\pm$4.21  & 227.94$\pm$54.91  & 31.06$\pm$13.55 & 257.98$\pm$51.74 & 24.94$\pm$8.34 \\
    HardMean                & 64.10$\pm$23.75   & 20.93$\pm$3.52 & 223.85$\pm$50.25  & 29.68$\pm$12.64 & 253.58$\pm$42.94 & 24.78$\pm$7.75 \\
    HardDense               & 61.82$\pm$23.06  & 20.09$\pm$2.93 & 170.41$\pm$47.02  & 23.61$\pm$11.20  & 243.63$\pm$39.60  & 24.43$\pm$6.23 \\
    \cellcolor{lightgray}{Mean}                    & \cellcolor{lightgray}{63.52$\pm$14.52}  & \cellcolor{lightgray}{21.53$\pm$2.75} & \cellcolor{lightgray}{246.57$\pm$23.84}  & \cellcolor{lightgray}{30.26$\pm$7.78}  & \cellcolor{lightgray}{\textbf{270.45$\pm$15.78}} & \cellcolor{lightgray}{\textbf{23.06$\pm$3.51}} \\ \hline
    \end{tabular}
    }
    \vspace{-0.2cm}
\end{table*}

\section{Experimental Evaluation}


In this section, we conduct comprehensive comparative and ablation experiments between CDCP and previous multi-task offline safe RL algorithms using a series of tasks with varying rewards, costs, observation states, and transition matrices.

\vspace{-0.2cm}
\subsection{Dataset and Baseline}
\noindent \textbf{Dataset.} 
To evaluate the performance of CDCP in the context of multi-tasking with safety constraints, we selected nine tasks commonly utilized in the safety-constrained domain as experimental tasks for this work. Concretely, we chose the \textit{EasySparse}, \textit{EasyMean}, \textit{EasyDense}, \textit{MediumSparse}, \textit{MediumMean}, \textit{MediumDense}, \textit{HardSparse}, \textit{HardMean}, and \textit{HardDense} tasks provided by the DSRL~\cite{liu2023datasets} dataset widely adopted in the field of safety RL~\cite{peng2021learning,peng2022safe}. These tasks include a variety of traffic scenarios and traffic flow densities, necessitating intelligent vehicles to follow traffic rules and drive safely to designated destinations. A detailed discussion of these tasks is provided in \textbf{Appendix~\ref{sup: experimental_task}}. Fig.~\ref{fig:data_episode} and~\ref{fig:data_step} illustrate the statistical features of the sample data for the selected tasks. Fig.~\ref{fig:data_episode} reveals notable variations in parameters such as maximum reward, maximum cost, mean reward, and mean cost across sample trajectories in diverse scenarios under different traffic flow densities.
Furthermore, Fig~\ref{fig:data_step} illustrates that the observed states and actions of samples across different scenarios exhibit significant differences under varying traffic flow densities. These results indicate that the selected tasks' observations, actions, rewards, and costs vary substantially. Consequently, the chosen autonomous driving tasks are well-suited for validating multi-task safe RL algorithms.

\noindent \textbf{Baseline.} Since no dedicated algorithms exist for multi-task offline safe RL, we enhance state-of-the-art offline safe RL and multi-task RL algorithms to develop multi-task offline safe RL algorithms as the experimental baseline. MTCPQ is an improved multi-task offline safe RL algorithm. It builds upon the single-task offline safe RL algorithm CPQ~\cite{xu2022constraints} by encoding the task ID into a feature vector using a multi-layer perceptron and incorporating this encoded vector into the observation state. This enhancement evolves CPQ into MTCPQ, extending its applicability from single-task to multi-task domains. Similar to MTCPQ, the MTCOptiDICE and MTCDT algorithms are improved multi-task offline safe RL algorithms. They expand the single-task offline safe RL algorithms COptiDICE~\cite{lee2021coptidice} and CDT~\cite{liu2023constrained} to the multi-task domain by encoding the task ID into a feature vector and incorporating this encoded vector into the observation state. MTCDD builds upon DD~\cite{ajay2022conditional} by encoding task IDs using a multi-layer perceptron and embedding them into observation states, thereby enhancing its applicability across multi-task domains. Moreover, it introduces cost constraints through classifier-free guidance, thereby expanding the utility of MTCDD into the multi-task safety domain. CMTDiff is an improved multi-task offline safe RL algorithm that extends upon MTDiff~\cite{he2023diffusion} by incorporating cost constraints through classifier-free guidance, thereby broadening its applicability to safety domains.
\begin{figure}[htp!]
    \centering
    \includegraphics[scale=0.285]{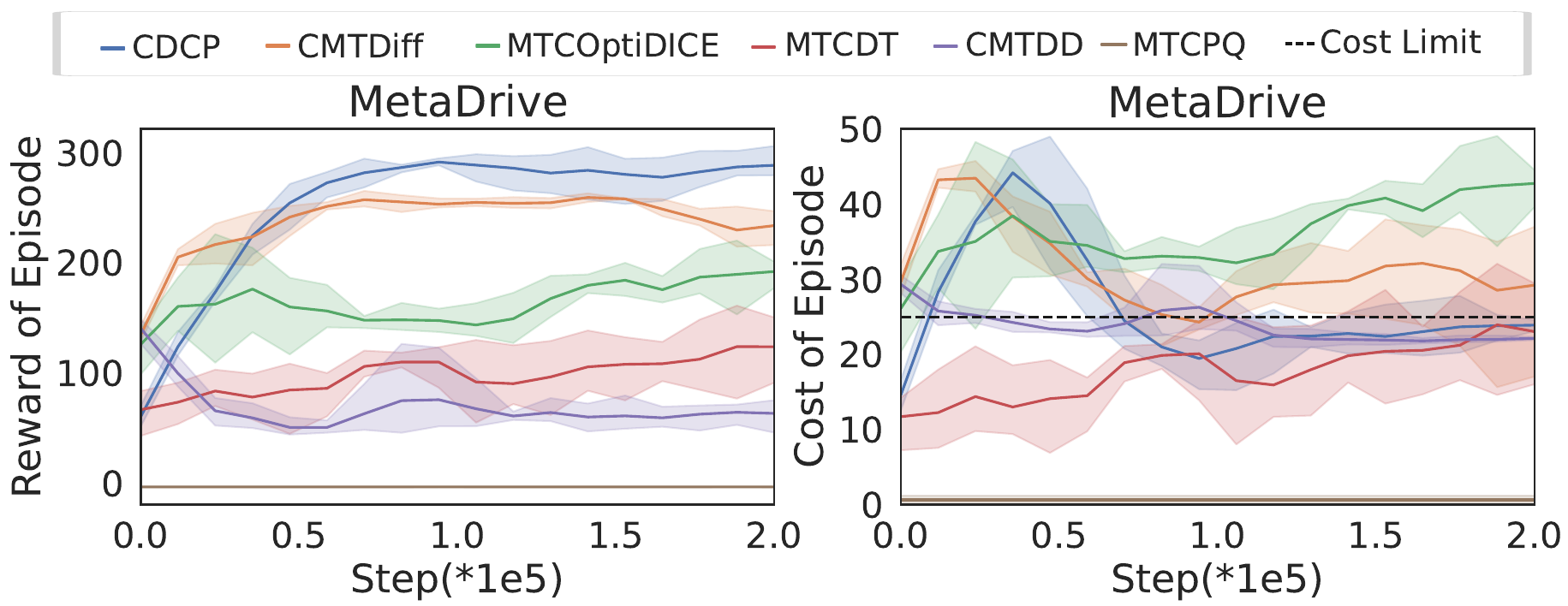}
    \vspace{-0.55cm}
    \caption{The rewards and costs of the CDCP and the current state-of-the-art multi-task offline safe RL algorithms. It presents the mean test results during the simultaneous training of nine tasks. The curves represent the average values of three random seeds, while the shaded areas indicate the standard deviation of the three random seeds.}
    \vspace{-0.35cm}
    \label{fig:comparison}
\end{figure}
\subsection{Comparison Experiment}
\textbf{The performance across all tasks.} We evaluate the performance of the CDCP algorithm in multi-task scenarios by conducting a comprehensive assessment across nine selected autonomous driving tasks. Additionally, we compare the performance of state-of-the-art multi-task offline safe RL algorithms on these autonomous driving tasks. Table~\ref{table: offline_Safe} illustrates a performance comparison between the CDCP algorithm and a multi-task offline safe RL algorithm improved through single-task offline safe RL. The results in the table show that the MTCDT algorithm offers compared to other baseline algorithms. However, it still cannot guarantee the safety of several tasks and achieves relatively lower reward returns. In contrast, the proposed CDCP algorithm incurs slightly higher costs for only a few tasks above the cost threshold and achieves significantly higher reward returns than other baseline methods. Additionally, the mean costs across multiple tasks remain within the threshold constraints. Table~\ref{table: mtrl} presents a performance comparison between the CDCP algorithm and  CMTDD and CMTDiff algorithms. The results indicate that the CMTDiff algorithm achieves relatively higher reward returns across several tasks than another baseline algorithm. However, it also incurs costs that exceed the threshold for most tasks, with the mean costs for all tasks surpassing the threshold. In contrast, the proposed CDCP algorithm keeps the costs for most tasks within the threshold and delivers higher reward returns than other baseline algorithms. Fig.~\ref{fig:comparison} illustrates the curves representing the mean reward and cost across all tasks during synchronous testing of the CDCP and the baseline algorithms throughout the training process. The figure's results indicate that the CDCP algorithm's mean cost stays below the threshold. At the same time, it delivers higher reward returns than other baseline algorithms. Based on the results and analysis presented above, it is clear that the CDCP algorithm achieves higher reward returns than other baseline algorithms while maintaining safety constraints. This underscores the superior performance of the CDCP algorithm relative to baseline methods, especially in achieving enhanced reward returns.
\begin{table*}[htp!]
    \centering
    \caption{The rewards and costs of the CDCP and state-of-the-art baseline algorithms under different cost thresholds. The table presents results obtained by training with three random seeds, recording the best model for each seed, and testing each best model across ten episodes to gather statistical data. The mean and standard deviation in the table represent the average rewards and costs computed after averaging all tasks for each random seed and episode. Additionally, the CDCP is trained solely with a cost threshold of 25, and no further training is performed when testing with other cost thresholds.}
    \vspace{-0.2cm}
    \label{table: various_cost_limit}
    \setlength{\tabcolsep}{2mm}{
    \begin{tabular}{ccccccccc}
    \hline
    \hline \\ [-6pt]
    \multirow{2}{*}{Method} & \multicolumn{2}{c}{30}     & \multicolumn{2}{c}{25}    & \multicolumn{2}{c}{20}    & \multicolumn{2}{c}{0}     \\ \cmidrule(lr{0pt}){2-3}  \cmidrule(lr{0pt}){4-5}  \cmidrule(lr{0pt}){6-7} \cmidrule(lr{0pt}){8-9} 
                            & Reward$\uparrow$            & Cost$\downarrow$           & Reward$\uparrow$            & Cost$\downarrow$          & Reward$\uparrow$            & Cost$\downarrow$          & Reward$\uparrow$            & Cost$\downarrow$          \\ \hline \\ [-6pt]
    MTCPQ                   & -3.42$\pm$0.36   & 0.98$\pm$0.47   & -3.28$\pm$0.42   & 0.63$\pm$0.41  & -3.72$\pm$0.24   & 1.28$\pm$0.09  & -2.45$\pm$1.74   & 0.87$\pm$0.62  \\
    MTCOptiDICE             & 188.27$\pm$39.33 & 41.95$\pm$8.67  & 178.34$\pm$34.72 & 39.79$\pm$7.51 & 162.81$\pm$33.74 & 33.17$\pm$7.05 & 125.16$\pm$24.82 & 21.94$\pm$8.16 \\
    MTCDT                   & 119.5$\pm$49.85  & 28.32$\pm$10.81 & 108.33$\pm$47.14 & 22.24$\pm$9.69 & 103.62$\pm$29.44 & 18.48$\pm$6.61 & 84.39$\pm$40.17  & 8.63$\pm$8.77  \\
    CMTDD                   & 87.94$\pm$17.96  & 26.97$\pm$3.16  & 63.52$\pm$14.52  & 21.53$\pm$2.75 & 50.12$\pm$10.67  & 17.35$\pm$2.31 & \textbf{28.76$\pm$7.35}   & \textbf{0.85$\pm$0.64}  \\
    CMTDiff                 & 278.36$\pm$31.02 & 33.85$\pm$9.13  & 246.57$\pm$23.84 & 30.26$\pm$7.78 & 223.16$\pm$20.13 & 26.39$\pm$6.74 & 168.93$\pm$14.58 & 13.38$\pm$4.17 \\
    \cellcolor{lightgray}{CDCP}                    & \cellcolor{lightgray}{\textbf{285.75$\pm$19.43}} & \cellcolor{lightgray}{\textbf{29.34$\pm$4.76}}  & \cellcolor{lightgray}{\textbf{270.45$\pm$15.78}} & \cellcolor{lightgray}{\textbf{23.06$\pm$3.51}} & \cellcolor{lightgray}{\textbf{254.14$\pm$14.32}} & \cellcolor{lightgray}{\textbf{19.93$\pm$3.26}} & \cellcolor{lightgray}{163.78$\pm$9.75}  & \cellcolor{lightgray}{1.28$\pm$0.93}  \\ \hline
    \end{tabular}
    }
\end{table*}

\noindent \textbf{The performance under various cost thresholds $\bar{c}_{i}^{\tau}$.} To assess the robustness of the CDCP and baseline algorithms across varying cost thresholds, we evaluate their performance under different cost thresholds. Table~\ref{table: various_cost_limit} displays the average reward and cost of the CDCP and baseline algorithms across all tasks at cost thresholds of $0$, $20$, $25$, and $30$. The table results demonstrate that at non-zero cost thresholds, the mean cost of the CDCP algorithm consistently remains within the cost constraints while delivering higher reward returns than the baseline algorithms. Additionally, although the CDCP algorithm does not achieve zero cost at a cost threshold, the cost is very close to zero and provides competitive reward returns. It is important to note that we evaluate the performance of the CDCP algorithm across various cost thresholds without retraining the model, adjusting the cost thresholds directly during the testing phase. This illustrates that the CDCP algorithm offers a flexible approach to configuring cost thresholds, consistent with Remark~\ref{rem: cost}. Furthermore, setting the cost threshold to zero transforms the CDCP algorithm from cumulative to state-wise cost constraints, consistent with Remark~\ref{rem:state-wise cost}. The above results and analysis indicate that the CDCP algorithm exhibits high robustness across different thresholds and provides a more flexible approach to setting cost thresholds.

\begin{table}[htp!]
    \centering
    \caption{The performances of the CDCP and CMTDiff algorithms under few-shot conditions. Six tasks are used as training, while the remaining three are designated as unseen tasks. The results are obtained by training with three random seeds, recording the best model for each seed, and testing each best model across ten episodes to gather statistical data. }
    \vspace{-0.20cm}
    \label{table: unseen_task}
    \setlength{\tabcolsep}{0.6mm}{
    \begin{tabular}{cccccc}
    \hline
    \hline \\ [-6pt]
    \multirow{2}{*}{Class}      & \multirow{2}{*}{Task} & \multicolumn{2}{c}{CMTDiff} & \multicolumn{2}{c}{CDCP} \\ \cmidrule(lr{0pt}){3-4} \cmidrule(lr{0pt}){5-6}
                                &                       & Reward $\uparrow$      & Cost $\downarrow$        & Reward $\uparrow$      & Cost$\downarrow$       \\ \hline \\ [-6pt]
    \multirow{6}{*}{\begin{tabular}[c]{@{}c@{}}Seen\\ Task\end{tabular}}   & EasSpa            & 337.79$\pm$65.70   & 42.52$\pm$20.48 & 328.06$\pm$50.69 & 25.03$\pm$4.75 \\
                                                                       & EasMea              & 328.96$\pm$49.84  & 37.07$\pm$15.69 & 321.17$\pm$47.09 & 24.9$\pm$4.99  \\
                                                                       & MedSpa          & 221.62$\pm$21.90   & 22.04$\pm$5.09  & 247.75$\pm$25.21 & 21.89$\pm$5.25 \\
                                                                       & MedMea            & 218.39$\pm$29.57  & 21.93$\pm$4.57  & 243.57$\pm$27.12 & 20.43$\pm$6.90  \\
                                                                       & HarSpa            & 229.50$\pm$52.44   & 26.4$\pm$11.32  & 262.35$\pm$50.44 & 24.47$\pm$8.14 \\
                                                                       & HarMea              & 226.48$\pm$49.86  & 24.92$\pm$10.07 & 258.15$\pm$40.69 & 24.11$\pm$7.02 \\ \hline \\ [-6pt]  
\multirow{3}{*}{\begin{tabular}[c]{@{}c@{}}Unseen\\ Task\end{tabular}} & EasDen             & 298.33$\pm$58.3   & 45.20$\pm$15.63  & 308.47$\pm$44.24 & 25.65$\pm$8.21 \\
                                                                       & MedDen           & 192.71$\pm$31.51  & 25.46$\pm$9.38  & 229.84$\pm$29.13 & 22.41$\pm$7.46 \\
                                                                       & HarDen             & 158.41$\pm$54.69  & 26.61$\pm$15.68 & 231.23$\pm$43.37 & 24.97$\pm$7.81 \\ \hline
    \end{tabular}
    }
\end{table}

\subsection{Few-shot Experiment}
To evaluate the generalization ability of the CDCP algorithm to unseen tasks, we further test its performance on unseen tasks with few-shot samples. We select six tasks on MetaDrive for training and evaluate the performance of the CDCP and MTDiff algorithms on three unseen tasks. The experimental results for unseen tasks with few-shot samples are presented in Table~\ref{table: unseen_task}. These results indicate that the CDCP algorithm ensures compliance with safety constraints and achieves superior performance compared to the baseline in the \textit{MediumDense} and \textit{HardDense} tasks. Furthermore, the cost in the \textit{EasyDense} task is nearly at the cost threshold. The results and analysis above demonstrate that the CDCP algorithm can generalize to unseen tasks with few-shot samples, thereby corroborating Remark~\ref{rem:zero-shot}.

\begin{figure}[htp!]
    \centering
    \includegraphics[scale=0.285]{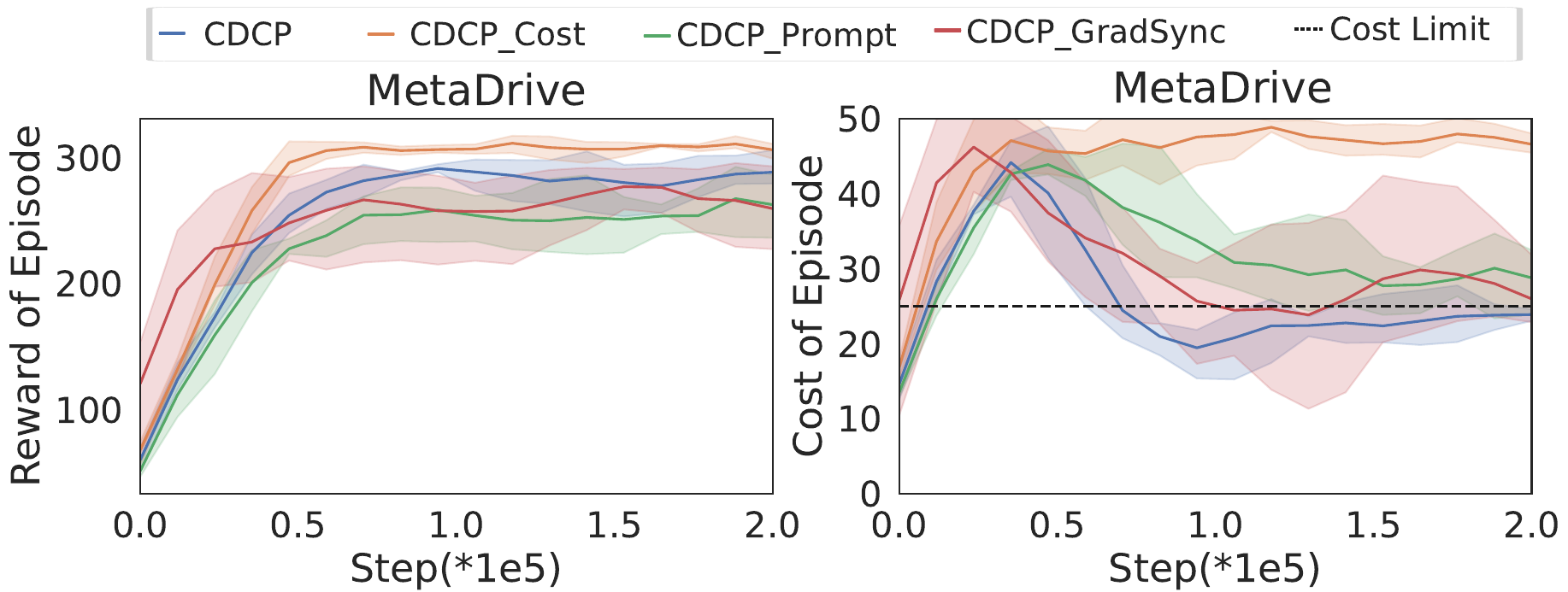}
    \vspace{-0.55cm}
    \caption{The reward and cost curves for the ablation experiments on the CDCP algorithm regarding cost constraints, contextual prompts, and gradient loss synchronization methods. It presents the mean test results during the simultaneous training of nine tasks. The curves represent the average values of three random seeds, while the shaded areas indicate the standard deviation of the three random seeds. }
    \vspace{-0.15cm}
    \label{fig:ablation_main}
\end{figure}

\begin{figure}[htp!]
    \centering
    \includegraphics[scale=0.285]{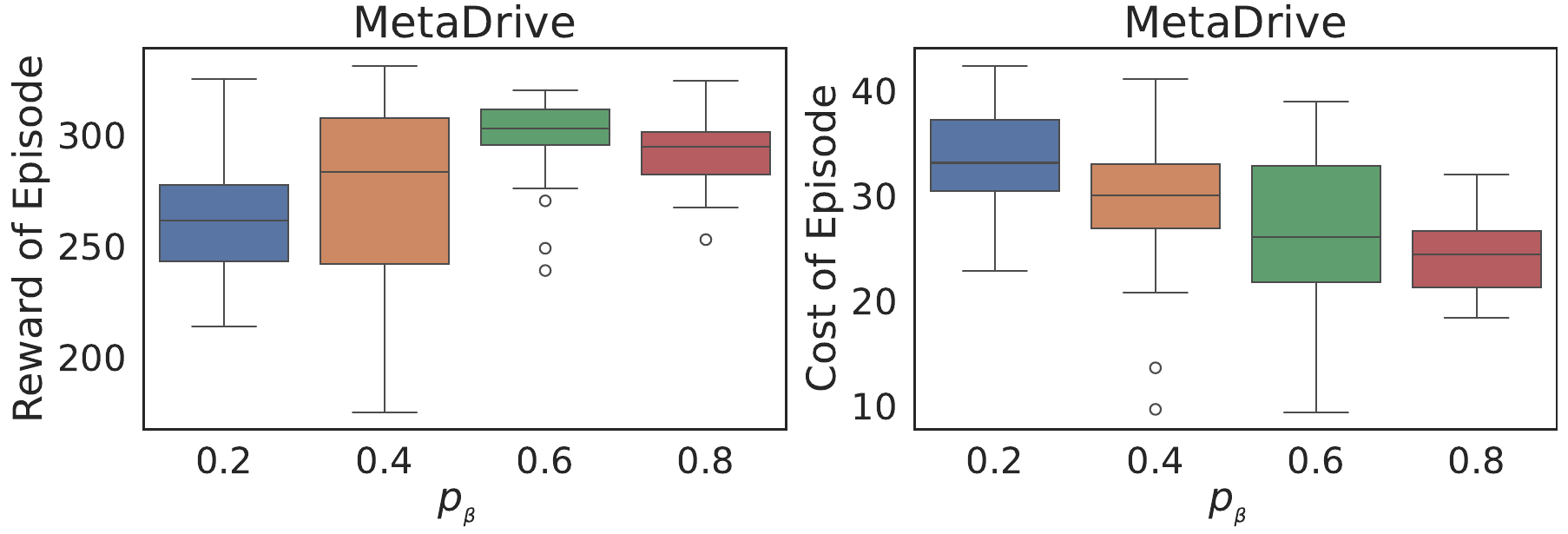}
    \vspace{-0.25cm}
    \caption{The mean and variance of rewards and costs for all tasks under different classifier-free guidance probability parameters $p_{_\beta}$ for the CDCP algorithm. The best model for each of the three random seeds is recorded during training, and each best model is subsequently tested across 10 episodes.}
    \vspace{-0.35cm}
    \label{fig:ablation_mask}
\end{figure}

\subsection{Ablation Experiment}
\noindent \textbf{Performance of Cost Constraints, Contextual Prompts, and Gradient Loss Synchronization.} We assess the impact of these methods on the CDCP algorithm by removing cost constraints, contextual prompts, and gradient loss synchronization methods. CDCP\_Cost represents the CDCP algorithm with ablate cost constraints, CDCP\_Prompt represents the CDCP algorithm with ablate contextual prompts, and CDCP\_GradSync represents the CDCP algorithm with ablate gradient loss synchronization. The ablation results shown in Fig.~\ref{fig:ablation_main} indicate that while the reward of the CDCP algorithm is lower than that of the CDCP\_Cost algorithm, the CDCP algorithm maintains its cost within the constraint threshold, whereas the cost of the CDCP\_Cost algorithm significantly exceeds this threshold. This suggests that the cost-constraint strategy of the CDCP algorithm effectively manages costs. Additionally, the results in Fig~\ref{fig:ablation_main} demonstrate that the reward of the CDCP\_Prompt algorithm is lower than that of the CDCP algorithm, and its cost sightly exceeds the cost threshold. This suggests that contextual prompt information can effectively improve the safety and reward return of the CDCP algorithm. The results in Fig.~\ref{fig:ablation_main} also show that the CDCP\_GradSync algorithm converges quickly but ultimately has a lower reward than the CDCP. Additionally, the reward and cost curves of the CDCP\_GradSync algorithm exhibit significant fluctuations, with the cost slightly exceeding the cost threshold. This indicates that the gradient loss synchronization strategy enhances the convergence stability of the CDCP algorithm and improves its performance by mitigating gradient interference. From the above results and analysis, we can conclude that both contextual prompts and gradient loss synchronization methods enhance the performance of the CDCP algorithm. Furthermore, the gradient synchronization method improves the convergence stability of the CDCP algorithm by mitigating gradient interference. Moreover, the cost constraint strategy effectively manages the cost of the CDCP algorithm.

\vspace{0.1cm}
\noindent \textbf{The parameter of classifier-free guidance probability $p_{_\beta}$. } 
To assess the impact of the unclassified guidance probability parameter on the CDCP algorithm, we evaluate its performance with $p_{_\beta}$ values set to 0.2, 0.4, 0.6, and 0.8, respectively. The results of the ablation experiments for parameter $p_{_\beta}$ are presented in Fig.~\ref{fig:ablation_mask}. The figure indicates that as parameter $p_{_\beta}$ increases, the reward of the CDCP algorithm also increases while the cost gradually decreases. However, once the parameter exceeds a certain threshold, the reward stops increasing and decreases while the cost stabilizes. This suggests that selecting an appropriate classifier-free guidance probability parameter $p_{_\beta}$ can effectively balance the reward and cost of the CDCP.

%% file: sec/6_conclusion.tex
\section{Conclusion}
In this work, we propose a novel conditional diffusion model with contextual prompts for the multi-task offline safe RL. Concretely, we reconsider the requirements of offline RL in practical applications and establish new objectives for multi-tasking offline safe RL. This objective focuses on learning policies that meet safety constraints using sample data from multiple offline tasks. Subsequently, we transform the contained optimization problem of multi-task offline safe RL into a conditional generation problem using the diffusion model. Building on this, we design a cost-constraint strategy with classifier-free guidance to eliminate extrapolation errors caused by temporal difference interactions through supervised learning. Additionally, we propose a contextual prompt method to achieve multi-task representation based on knowledge sharing. Furthermore, to address interference between tasks, we introduce a gradient loss synchronization strategy to eliminate gradient interference. Extensive experiments in the MetaDrive environment demonstrate that the CDCP algorithm provides higher reward returns than state-of-the-art baseline methods while ensuring the safety of multiple tasks.

%% file: sec/7_supplement.tex
\clearpage
\section{Supplementary Material}
\subsection{Proof and Discussion the Proposition IV.2}
\label{sup: class-free}
\begin{proposition}
\label{pro:class-free_s}
The perturbation noise $\hat{\epsilon}$ in the intermediate steps of the denoising process, which takes the state sequence $\mathcal{O}(\tau_{c})$ and task contextual information $\mathcal{M}(\tau)$ as necessary conditions and employs classifier-free guidance to constrain cost and maximize reward, is defined as:
\begin{equation} 
    \begin{aligned}
        \label{equ_s3}
        &\hat{\epsilon}=\kappa\Big[\epsilon_{\theta}\big(x_{k}(\tau_{c}),y'(\tau_{c},\tau),\mathcal{R}(\tau_{c}),k\big) \\
        & \quad +\epsilon_{\theta}\big(x_{k}(\tau_{c}),y'(\tau_{c},\tau),\mathcal{C}(\tau_{c}),k\big)\Big]\\
        &\quad +(1-2\kappa)\epsilon_{\theta}\big(x_{k}(\tau_{c}),y'(\tau_{c},\tau),\varnothing,k\big),
    \end{aligned}
\end{equation}
\noindent where, $y'(\tau_{c},\tau)$ represents the state sequence and task contextual information $y'(\tau_{c},\tau)= (\mathcal{O}(\tau_{c}),\mathcal{M}(\tau))$. $\mathcal{R}$ and $\mathcal{C}$ are the reward and cost conditions, $\kappa$ is the classifier-free guidance coefficient, and $\epsilon_{\theta}$ denotes the parameterized noise distribution with parameters $\theta$.
\end{proposition} 
\noindent \textbf{\textit{Proof.}} We assume that $y'(\tau_c,\tau)$, $\mathcal{R}(\tau_c)$, and $\mathcal{C}(\tau_c)$ are conditionally independent given $x_{k}(\tau_c)$. Given this assumption, the derivation proceeds as follows:
\begin{equation} 
    \begin{aligned}
        \label{equ_15_s}       
        &p\big(x_{k}|y',\mathcal{R},\mathcal{C}\big) \overset{\{i\}}{\propto} p(x_{k}|y') \frac{p(x_{k}|y',\mathcal{R})}{p(x_{k}|y')}\frac{p(x_{k}|y',\mathcal{C})}{p(x_{k}|y')} \\
        & \quad \quad \quad \quad \quad \overset{\{ii\}}{=} p(x_{k}|y') \left[\frac{p(x_{k}|y',\mathcal{R})}{p(x_{k}|y')}\frac{p(x_{k}|y',\mathcal{C})}{p(x_{k}|y')}\right]^{\kappa},\\        
    \end{aligned}
\end{equation}
\noindent where $y'$,$\mathcal{R}$, and $ \mathcal{C}$ are abbreviations for $y'(\tau_c,\tau)$,$\mathcal{R}(\tau_c)$ and $\mathcal{C}(\tau_c)$, respectively. 
$\{i\}$ indicates that event $a$ is independent of event $b$, and we have the equation $p(a|b)=p(a)$. $\{ii\}$ indicates that the $k$-th power of 1 is still equal to 1. Taking the logarithm of Eq.~\ref{equ_15_s} and then taking the partial derivative concerning $x_k$, we obtain Eq.~\ref{equ_16_s} and~\ref{equ_17_s}. 
\begin{equation} 
    \begin{aligned}
        \label{equ_16_s}        
        &\log{p\big(x_{k}|y',\mathcal{R},\mathcal{C}\big)} \propto \log{p(x_{k}|y')} \\
        & \quad \quad \quad \quad ~~~ + \log \left[\frac{p(x_{k}|y',\mathcal{R})}{p(x_{k}|y')}\frac{p(x_{k}|y',\mathcal{C})}{p(x_{k}|y')}\right]^{\kappa}\\
        & = \log{p(x_{k}|y')} +      \kappa \log\left[\frac{p(x_{k}|y',\mathcal{R})}{p(x_{k}|y')}\frac{p(x_{k}|y',\mathcal{C})}{p(x_{k}|y')}\right]\\
        & = \log{p(x_{k}|y')}+ 
       \kappa \big[\log(p(x_{k}|y',\mathcal{R}))- \log{p(x_{k}|y')}\big] \\
        &\quad \quad \quad\quad \quad ~~ +\kappa\big[\log(p(x_{k}|y',\mathcal{C}))-\log{p(x_{k}|y')}\big].      
    \end{aligned}
\end{equation}

\begin{equation} 
    \begin{aligned}
        \label{equ_17_s}        &\bigtriangledown_{x_{k}}\log{p\big(x_{k}|y',\mathcal{R},\mathcal{C}\big)} \propto \bigtriangledown_{x_{k}}\log{p(x_{k}|y')} \\
        & \quad \quad   +\kappa\left[\bigtriangledown_{x_{k}}\log(p(x_{k}|y',\mathcal{R})- \bigtriangledown_{x_{k}}\log{p(x_{k}|y')}\right] \\
        & \quad \quad  +\kappa[\bigtriangledown_{x_{k}}\log(p(x_{k}|y',\mathcal{C}))-\bigtriangledown_{x_{k}}\log{p(x_{k}|y')}].
    \end{aligned}
\end{equation}
Building on the conclusion of previous work~\cite{luo2022understanding,songdenoising} that $\bigtriangledown_{x_{k}}\log{p\big(x_{k}|y',\mathcal{R},\mathcal{C}\big)}\propto-\epsilon_{\theta}(x_{k},y',\mathcal{R},\mathcal{C},k)$, we derive multiple conditional for the perturbation noise from Eq.~\ref{equ_17_s}, as shown below:
\begin{equation} 
    \begin{aligned}
        \label{equ_18_s}        
        &\epsilon_{\theta}(x_{k},y'(\tau_c,\tau),\mathcal{R},\mathcal{C},k) \propto \epsilon_{\theta}(x_{k},y',\varnothing,k)\\
        & \quad \quad \quad \quad \quad \quad  + \kappa[\epsilon_{\theta}(x_{k},y',\mathcal{R},k)-\epsilon_{\theta}(x_{k},y',\varnothing,k)] \\
        & \quad \quad \quad \quad \quad \quad + \kappa[\epsilon_{\theta}(x_{k},y',\mathcal{C},k)-\epsilon_{\theta}(x_{k},y',\varnothing,k)].
    \end{aligned}
\end{equation}

Based on the above derivations and analysis, we employ the following perturbation noise to sample from $q\big(x_{k}(\tau_c)|y'(\tau_c,\tau),\mathcal{R}(\tau_c),\mathcal{C}(\tau_c)\big) $ under classifier-free for the reward and cost conditions:
\begin{equation} 
    \begin{aligned}
        \label{equ_19_s}
        &\hat{\epsilon}=\kappa\Big[\epsilon_{\theta}\big(x_{k}(\tau_{c}),y'(\tau_{c},\tau),\mathcal{R}(\tau_{c}),k\big) \\
        & \quad +\epsilon_{\theta}\big(x_{k}(\tau_{c}),y'(\tau_{c},\tau),\mathcal{C}(\tau_{c}),k\big)\Big]\\
        &\quad +(1-2\kappa)\epsilon_{\theta}\big(x_{k}(\tau_{c}),y'(\tau_{c},\tau),\varnothing,k\big).
    \end{aligned}
\end{equation}
In conclusion, Proposition~\ref{pro:class-free_s} is proven.

\begin{table}[htbp!] \small
\centering
\caption{The primary network parameters and hyperparameters of the CDCP algorithm's model.}  
\setlength{\tabcolsep}{5mm}{
    \begin{tabular}{cc}
    \hline
    \hline
    Sort                                                                                           & Setting               \\ \hline
    \multirow{9}{*}{Parameters}                                                              & $\gamma$=0.99            \\
                                                                                                   & batch\_size=64        \\
                                                                                                   & $\lambda$=2e-4   \\
                                                                                                   & $\eta$=0.12              \\
                                                                                                   & $K$=200                 \\
                                                                                                   & $h$=32                  \\
                                                                                                   & $\chi$=20                  \\
                                                                                                   & $l$ =2 \\
                                                                                                   & $p_{_\beta}\!\!=\!0.8$ \\
                                                                                                   \\[-8pt]
                                                                                                   \hline
                                                                                                   \\[-8pt]
    \multirow{7}{*}{\begin{tabular}[c]{@{}c@{}}Transformer for Diffusion\\ GPT2Model\end{tabular}} & n\_layer=4            \\
                                                                                                   & n\_head=2             \\
                                                                                                   & n\_inner=1024         \\
                                                                                                   & n\_positon=1024       \\
                                                                                                   & n\_ctx=1023           \\
                                                                                                   & reside\_pdrop=0.1     \\
                                                                                                   & attn\_pdrop=0.1       \\
                                                                                                   \\[-8pt]
                                                                                                   \hline
                                                                                                   \\ [-8pt]
    Pre-trained LLM Model                                                                          & Bertmodel             \\
    MLP for K Embedding                                                                            & {[}256,512,256{]} \\
    MLP for R Embedding                                                                            & {[}1,256,512,256{]}   \\
    MLP for C Embedding                                                                            & {[}1,256,512,256{]}   \\
    MLP for Text Embedding                                                                         & {[}4,256,512,256{]}   \\
    MLP for Prom Embedding                                                                         & {[}261,256,512,256{]} \\
    MLP for Obs Embedding                                                                          & {[}259,256,512,256{]} \\ \hline
    \end{tabular}
    }
\label{tab:hyper-parameters}
\end{table}

\subsection{Experimental Details}
\label{sup: experimental}
\subsubsection{Experimental Setting}
The experiment is conducted on a server equipped with L40s, utilizing PyTorch for model construction, training, and testing. Table 2 outlines the primary network parameters of the model implemented by our CDCP algorithm. Furthermore, the code appendix contains detailed configuration information for the testing environment. For instructions on safely setting up and configuring the training and testing environment for the CDCP model, please refer to the README file in the appendix.

\begin{figure*}[htp!]
  \centering
  \begin{subfigure}[b]{0.8\textwidth}
    \includegraphics[width=\textwidth]{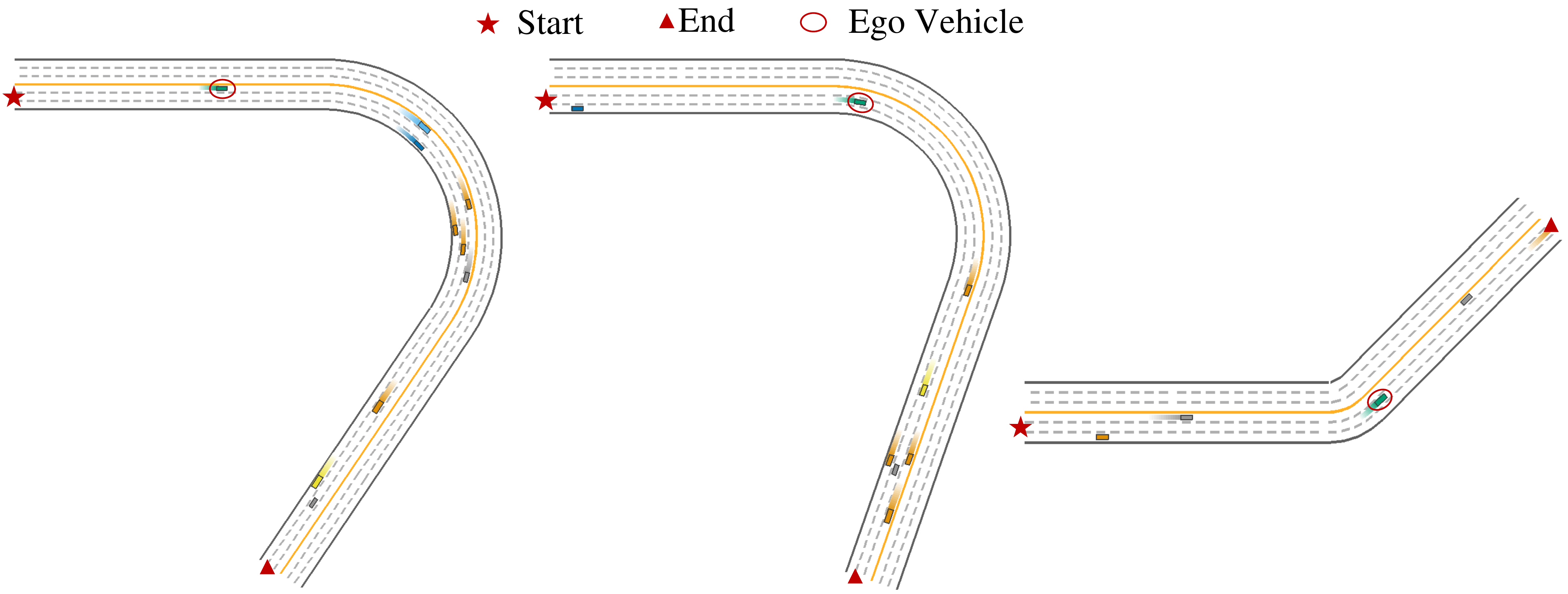}
    \caption{Sparse}
  \end{subfigure}%
  
  \begin{subfigure}[b]{0.8\textwidth}
    \includegraphics[width=\textwidth]{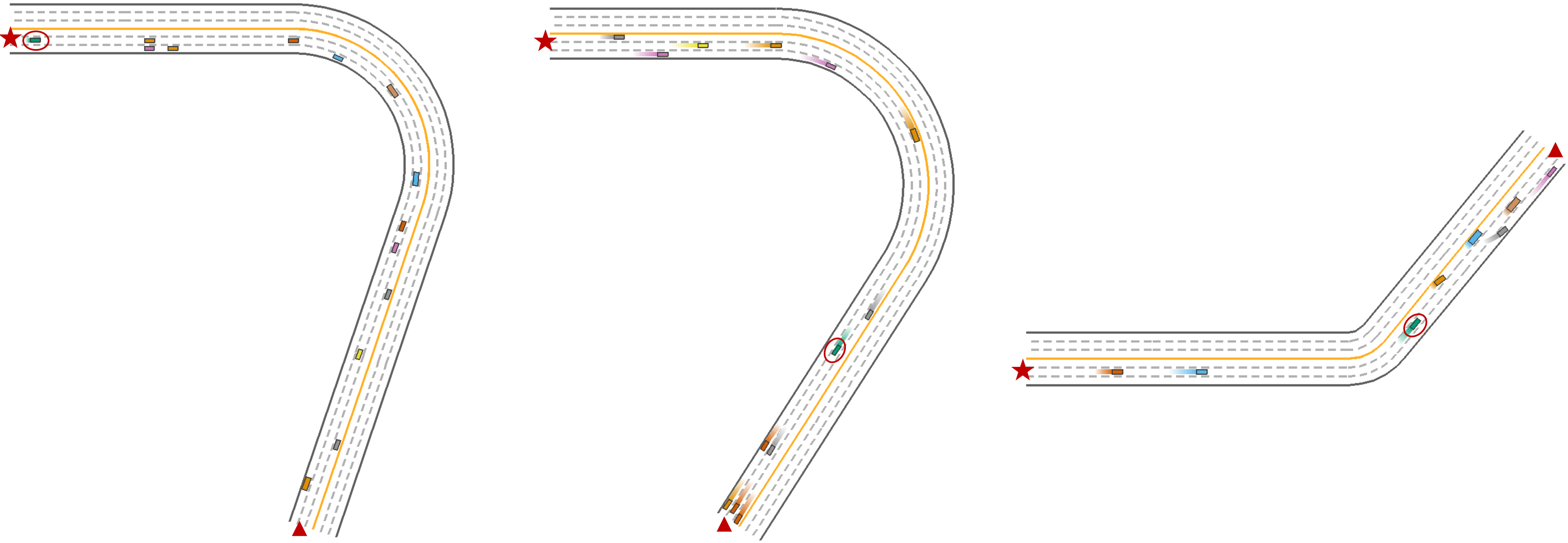}
    \caption{Mean}
  \end{subfigure}%
  
  \begin{subfigure}[b]{0.8\textwidth}
    \includegraphics[width=\textwidth]{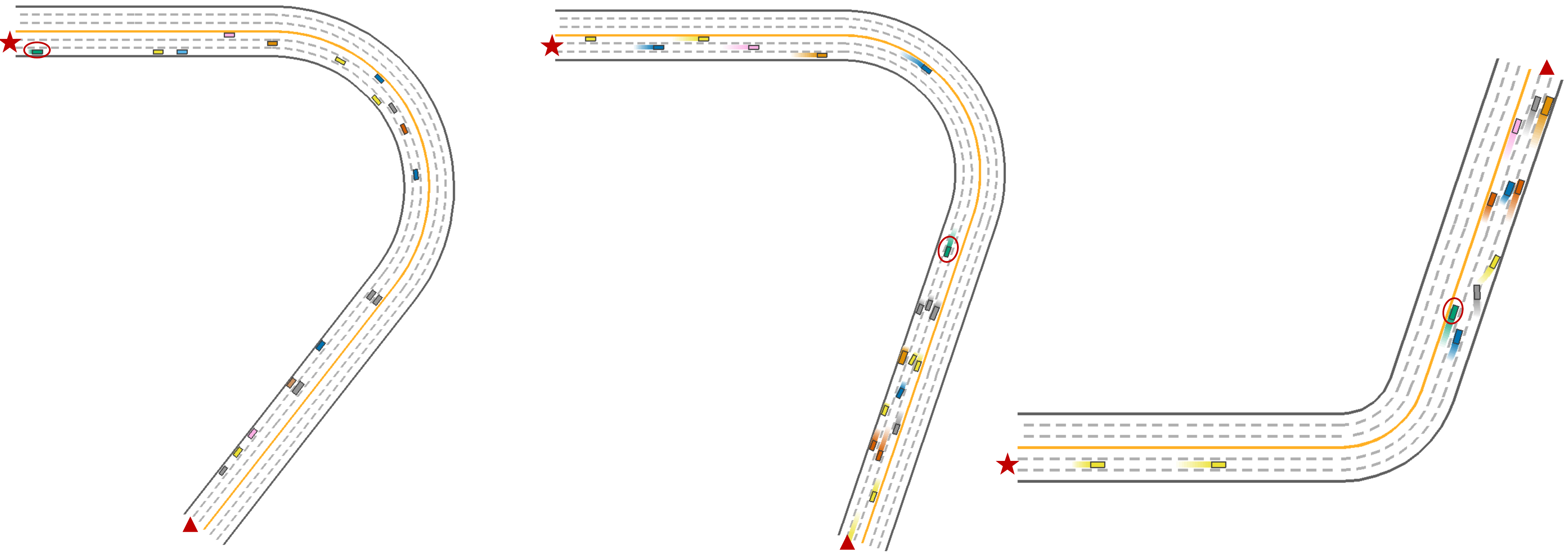}
    \caption{Dense}
  \end{subfigure}
  \caption{
  Visualization of driving scenarios with different traffic densities in the \textit{easy} traffic scenario. The pentagram represents the road where the ego vehicle is generated; the vehicle within the oval region is the ego vehicle, and the triangle indicates the exit for the ego vehicle. (a) Visualized traffic scenarios under \textit{sparse} traffic flow, (b) visualized traffic scenarios under \textit{mean} traffic flow, and (c) visualized traffic scenarios under \textit{dense} traffic flow.
}
  \vspace{-0.2cm}
  \label{fig: easy}
\end{figure*}

\subsubsection{Tasks for Experimental}
\label{sup: experimental_task}
Our experiments select nine safety-constrained autonomous driving scenarios encompassing diverse driving situations and traffic densities. These scenarios include environments such as curves, intersections, T-junctions, and roundabouts. Regarding traffic density, the tasks cover sparse, medium, and dense traffic flow. Fig.~\ref{fig: easy} presents the visualization results of the traffic scenarios under \textit{sparse}, \textit{mean}, and \textit{dense} traffic flows within the easy scenario. This \textit{easy} scenario features curves with varying curvatures, where the ego vehicle is generated as one end of the road with random speed and lane assignment. The objective of the autonomous vehicle in this scenario is to reach the other end of the road without collisions while adhering to traffic rules at an appropriate speed. Fig.~\ref{fig: medium} shows the visualization results of the traffic scenarios under \textit{sparse}, \textit{mean}, \textit{dense} traffic flows in the \textit{medium} scenario. Those scenarios include T-junctions and cross-intersections, with the ego vehicle generated at multiple intersections with random speed and lane assignment. The objective of the ego vehicle in this scenario is to pass through one or more intersections without collisions and reach the exit at an appropriate speed. Fig.~\ref{fig: hard} illustrates the visualization of traffic scenarios under varying traffic densities in the \textit{dense} scenario. This scenario features T-junctions, roundabouts, and ramps, where the ego vehicle is generated at multiple intersections with random speeds and lane assignments. The goal for the autonomous vehicle in this scenario is to safely navigate through the intersections and roundabouts at an appropriate speed and reach the exit without collisions.

Additionally, the results depicted in Fig.~\ref{fig: easy}, \ref{fig: medium}, and \ref{fig: hard} indicated that dense traffic flow has significantly more vehicles in the same traffic scenarios compared to sparse traffic flow with medium traffic flow having a moderate number of traffic vehicles. Furthermore, the vehicles in all nine scenarios display intelligent behaviors, including autonomous lane changing and acceleration/deceleration.

\begin{figure*}[htp!]
  \centering
  \begin{subfigure}[b]{1\textwidth}
    \includegraphics[width=\textwidth]{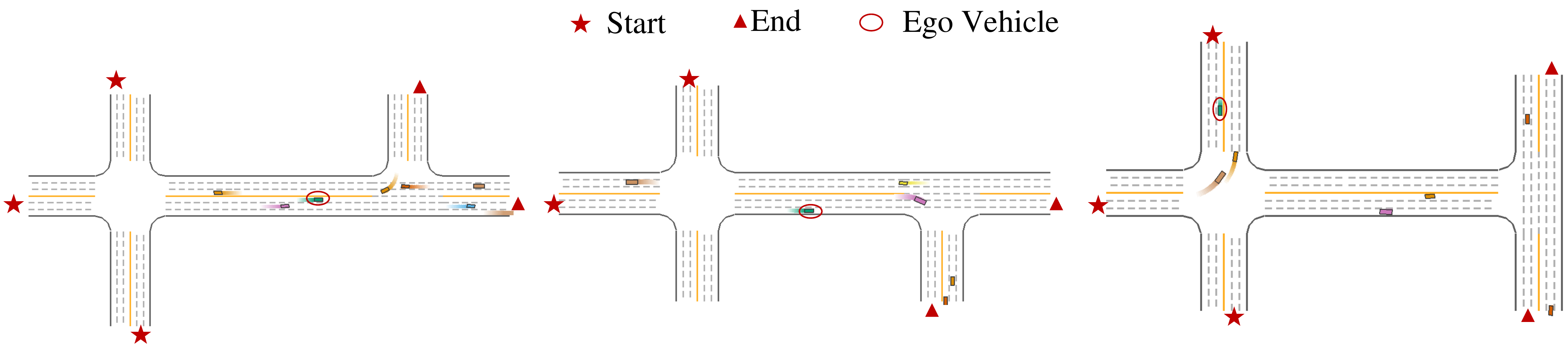}
    \caption{Sparse}
  \end{subfigure}%
  
  \begin{subfigure}[b]{1\textwidth}
    \includegraphics[width=\textwidth]{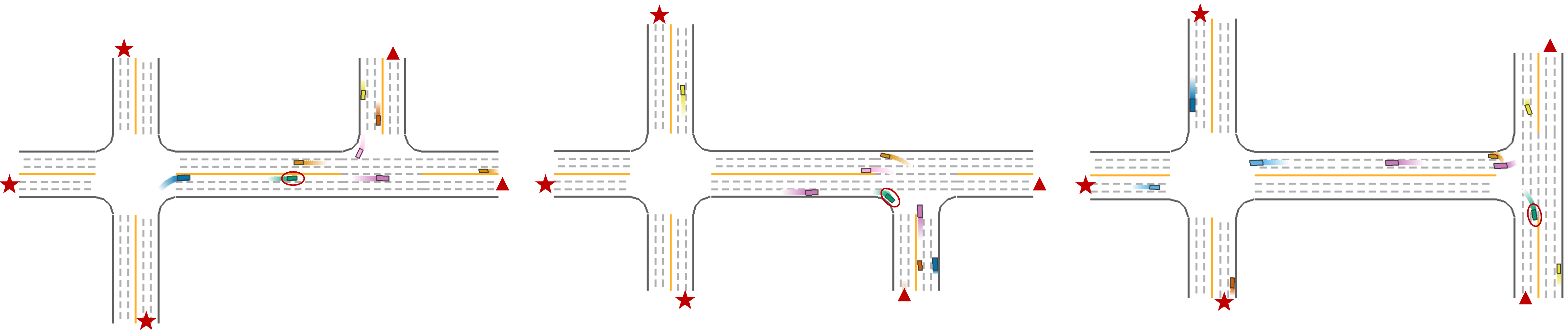}
    \caption{Mean}
  \end{subfigure}%
  
  \begin{subfigure}[b]{1\textwidth}
    \includegraphics[width=\textwidth]{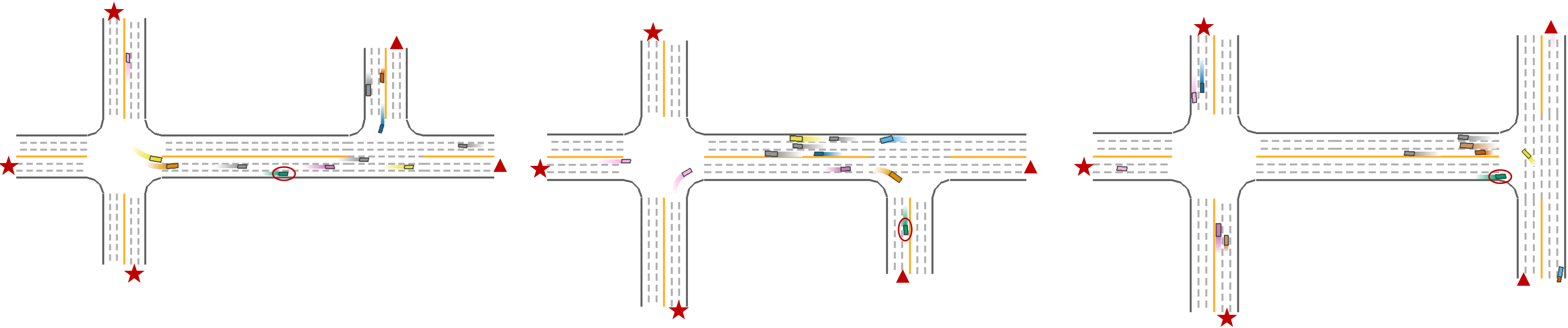}
    \caption{Dense}
  \end{subfigure}
  \caption{Visualization of driving scenarios with different traffic densities in the \textit{medium} traffic scenario. The pentagram represents the road where the ego vehicle is generated; the vehicle within the oval region is the ego vehicle, and the triangle indicates the exit for the ego vehicle. (a) Visualized traffic scenarios under \textit{sparse} traffic flow, (b) visualized traffic scenarios under \textit{mean} traffic flow, and (c) visualized traffic scenarios under \textit{dense} traffic flow.}
  \vspace{-0.2cm}
  \label{fig: medium}
\end{figure*}

\begin{figure*}[htp!]
  \centering
  \begin{subfigure}[b]{1\textwidth}
    \includegraphics[width=\textwidth]{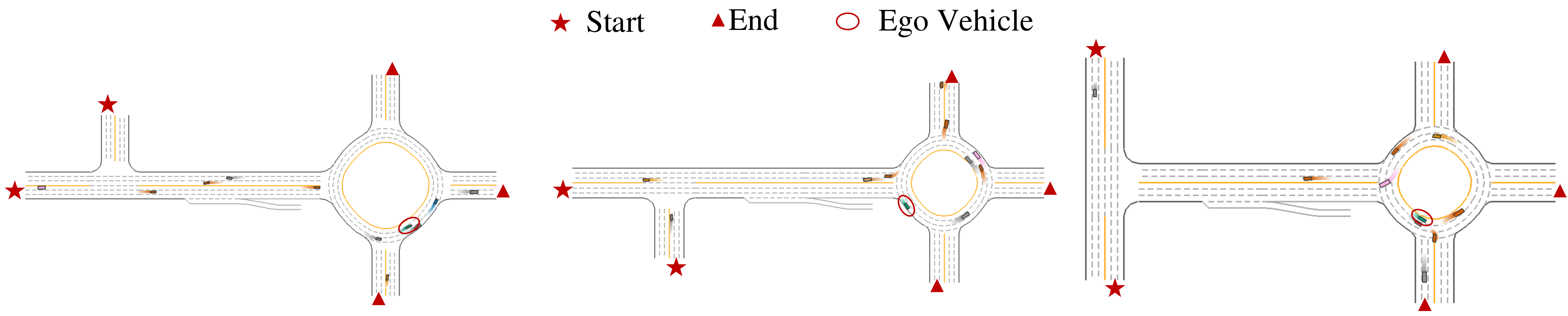}
    \caption{Sparse}
  \end{subfigure}%
  
  \begin{subfigure}[b]{1\textwidth}
    \includegraphics[width=\textwidth]{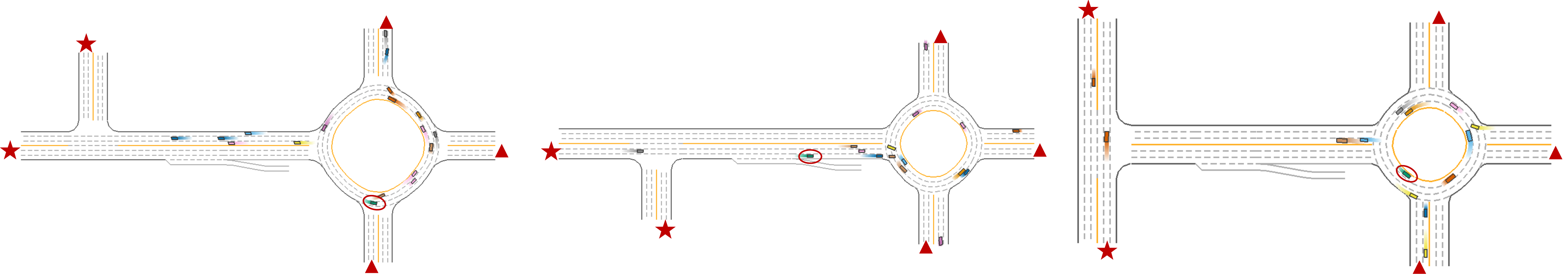}
    \caption{Mean}
  \end{subfigure}%
  
  \begin{subfigure}[b]{1\textwidth}
    \includegraphics[width=\textwidth]{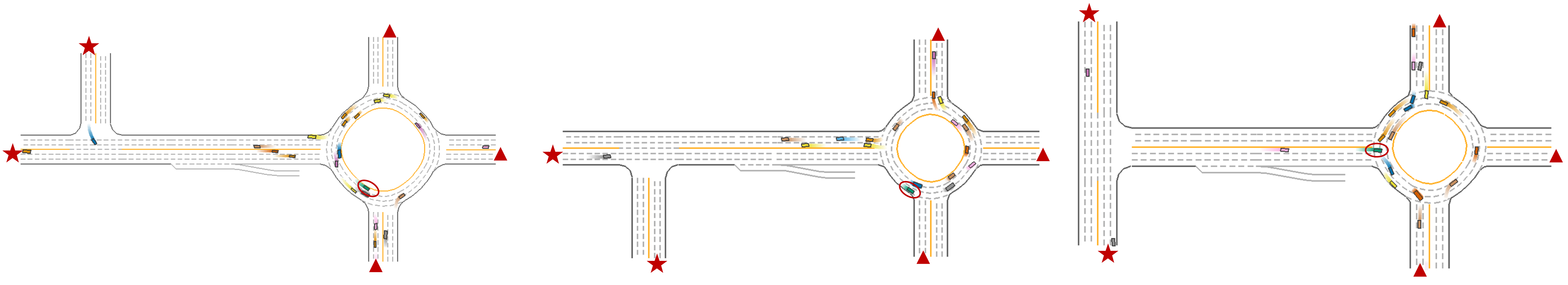}
    \caption{Dense}
  \end{subfigure}
  \caption{Visualization of driving scenarios with different traffic densities in the \textit{hard} traffic scenario. The pentagram represents the road where the ego vehicle is generated; the vehicle within the oval region is the ego vehicle, and the triangle indicates the exit for the ego vehicle. (a) Visualized traffic scenarios under \textit{sparse} traffic flow, (b) visualized traffic scenarios under \textit{mean} traffic flow, and (c) visualized traffic scenarios under \textit{dense} traffic flow.}
  \vspace{-0.2cm}
  \label{fig: hard}
\end{figure*}

%% file: ref.bib
@article{nie2018longitudinal,
  title={Longitudinal speed control of autonomous vehicle based on a self-adaptive PID of radial basis function neural network},
  author={Nie, Linzhen and Guan, Jiayi and Lu, Chihua and Zheng, Hao and Yin, Zhishuai},
  journal={IET Intelligent Transport Systems},
  volume={12},
  number={6},
  pages={485--494},
  year={2018},
  publisher={Wiley Online Library}
}

@article{guan2022discrete,
  title={A discrete soft actor-critic decision-making strategy with sample filter for freeway autonomous driving},
  author={Guan, Jiayi and Chen, Guang and Huang, Jin and Li, Zhijun and Xiong, Lu and Hou, Jing and Knoll, Alois},
  journal={IEEE Transactions on Vehicular Technology},
  volume={72},
  number={2},
  pages={2593--2598},
  year={2022},
  publisher={IEEE}
}

@article{li2022metadrive,
  title={Metadrive: Composing diverse driving scenarios for generalizable reinforcement learning},
  author={Li, Quanyi and Peng, Zhenghao and Feng, Lan and Zhang, Qihang and Xue, Zhenghai and Zhou, Bolei},
  journal={IEEE transactions on pattern analysis and machine intelligence},
  volume={45},
  number={3},
  pages={3461--3475},
  year={2022},
  publisher={IEEE}
}

@article{huang2022curriculum,
  title={Curriculum-based asymmetric multi-task reinforcement learning},
  author={Huang, Hanchi and Ye, Deheng and Shen, Li and Liu, Wei},
  journal={IEEE transactions on pattern analysis and machine intelligence},
  volume={45},
  number={6},
  pages={7258--7269},
  year={2022},
  publisher={IEEE}
}

@article{yang2020multi,
  title={Multi-task reinforcement learning with soft modularization},
  author={Yang, Ruihan and Xu, Huazhe and Wu, Yi and Wang, Xiaolong},
  journal={Advances in Neural Information Processing Systems},
  volume={33},
  pages={4767--4777},
  year={2020}
}

@article{sun2022paco,
  title={Paco: Parameter-compositional multi-task reinforcement learning},
  author={Sun, Lingfeng and Zhang, Haichao and Xu, Wei and Tomizuka, Masayoshi},
  journal={Advances in Neural Information Processing Systems},
  volume={35},
  pages={21495--21507},
  year={2022}
}

@ARTICLE{zhang2022residual,
  author={Zhang, Ruiqi and Hou, Jing and Chen, Guang and Li, Zhijun and Chen, Jianxiao and Knoll, Alois},
  journal={IEEE Robotics and Automation Letters}, 
  title={Residual Policy Learning Facilitates Efficient Model-Free Autonomous Racing}, 
  year={2022},
  volume={7},
  number={4},
  pages={11625-11632},
  doi={10.1109/LRA.2022.3192770}}

@article{chen2020just,
  title={Just pick a sign: Optimizing deep multitask models with gradient sign dropout},
  author={Chen, Zhao and Ngiam, Jiquan and Huang, Yanping and Luong, Thang and Kretzschmar, Henrik and Chai, Yuning and Anguelov, Dragomir},
  journal={Advances in Neural Information Processing Systems},
  volume={33},
  pages={2039--2050},
  year={2020}
}

@article{yu2021conservative,
  title={Conservative data sharing for multi-task offline reinforcement learning},
  author={Yu, Tianhe and Kumar, Aviral and Chebotar, Yevgen and Hausman, Karol and Levine, Sergey and Finn, Chelsea},
  journal={Advances in Neural Information Processing Systems},
  volume={34},
  pages={11501--11516},
  year={2021}
}

@article{liu2021conflict,
  title={Conflict-averse gradient descent for multi-task learning},
  author={Liu, Bo and Liu, Xingchao and Jin, Xiaojie and Stone, Peter and Liu, Qiang},
  journal={Advances in Neural Information Processing Systems},
  volume={34},
  pages={18878--18890},
  year={2021}
}

@inproceedings{sodhani2021multi,
  title={Multi-task reinforcement learning with context-based representations},
  author={Sodhani, Shagun and Zhang, Amy and Pineau, Joelle},
  booktitle={International Conference on Machine Learning},
  pages={9767--9779},
  year={2021},
  organization={PMLR}
}

@article{cheng2023multi,
  title={Multi-task reinforcement learning with attention-based mixture of experts},
  author={Cheng, Guangran and Dong, Lu and Cai, Wenzhe and Sun, Changyin},
  journal={IEEE Robotics and Automation Letters},
  year={2023},
  publisher={IEEE}
}

@article{adjei2024safe,
  title={Safe Reinforcement Learning for Arm Manipulation with Constrained Markov Decision Process},
  author={Adjei, Patrick and Tasfi, Norman and Gomez-Rosero, Santiago and Capretz, Miriam AM},
  journal={Robotics},
  volume={13},
  number={4},
  pages={63},
  year={2024},
  publisher={MDPI}
}

@article{wachi2024safe,
  title={Safe exploration in reinforcement learning: A generalized formulation and algorithms},
  author={Wachi, Akifumi and Hashimoto, Wataru and Shen, Xun and Hashimoto, Kazumune},
  journal={Advances in Neural Information Processing Systems},
  volume={36},
  year={2024}
}

@article{he2024toward,
  title={Toward trustworthy decision-making for autonomous vehicles: A robust reinforcement learning approach with safety guarantees},
  author={He, Xiangkun and Huang, Wenhui and Lv, Chen},
  journal={Engineering},
  volume={33},
  pages={77--89},
  year={2024},
  publisher={Elsevier}
}

@article{zhang2021provably,
  title={Provably efficient multi-task reinforcement learning with model transfer},
  author={Zhang, Chicheng and Wang, Zhi},
  journal={Advances in Neural Information Processing Systems},
  volume={34},
  pages={19771--19783},
  year={2021}
}

@article{chen2021multi,
  title={A multi-task-learning-based transfer deep reinforcement learning design for autonomic optical networks},
  author={Chen, Xiaoliang and Proietti, Roberto and Liu, Che-Yu and Yoo, SJ Ben},
  journal={IEEE Journal on Selected Areas in Communications},
  volume={39},
  number={9},
  pages={2878--2889},
  year={2021},
  publisher={IEEE}
}

@inproceedings{hessel2019multi,
  title={Multi-task deep reinforcement learning with popart},
  author={Hessel, Matteo and Soyer, Hubert and Espeholt, Lasse and Czarnecki, Wojciech and Schmitt, Simon and Van Hasselt, Hado},
  booktitle={Proceedings of the AAAI Conference on Artificial Intelligence},
  volume={33},
  number={01},
  pages={3796--3803},
  year={2019}
}

@article{tomov2021multi,
  title={Multi-task reinforcement learning in humans},
  author={Tomov, Momchil S and Schulz, Eric and Gershman, Samuel J},
  journal={Nature Human Behaviour},
  volume={5},
  number={6},
  pages={764--773},
  year={2021},
  publisher={Nature Publishing Group UK London}
}

@article{parvini2023aoi,
  title={AoI-aware resource allocation for platoon-based C-V2X networks via multi-agent multi-task reinforcement learning},
  author={Parvini, Mohammad and Javan, Mohammad Reza and Mokari, Nader and Abbasi, Bijan and Jorswieck, Eduard A},
  journal={IEEE Transactions on Vehicular Technology},
  volume={72},
  number={8},
  pages={9880--9896},
  year={2023},
  publisher={IEEE}
}

@article{li2020multi,
  title={Multi-task batch reinforcement learning with metric learning},
  author={Li, Jiachen and Vuong, Quan and Liu, Shuang and Liu, Minghua and Ciosek, Kamil and Christensen, Henrik and Su, Hao},
  journal={Advances in Neural Information Processing Systems},
  volume={33},
  pages={6197--6210},
  year={2020}
}

@inproceedings{d2020sharing,
  title={Sharing Knowledge in Multi-Task Deep Reinforcement Learning},
  author={D'Eramo, C and Tateo, D and Bonarini, A and Restelli, M and Peters, J and others},
  booktitle={8th International Conference on Learning Representations,$\{$ICLR$\}$ 2020, Addis Ababa, Ethiopia, April 26-30, 2020},
  pages={1--11},
  year={2020},
  organization={International Conference on Learning Representations, ICLR}
}

@article{lee2022coptidice,
  title={Coptidice: Offline constrained reinforcement learning via stationary distribution correction estimation},
  author={Lee, Jongmin and Paduraru, Cosmin and Mankowitz, Daniel J and Heess, Nicolas and Precup, Doina and Kim, Kee-Eung and Guez, Arthur},
  journal={arXiv preprint arXiv:2204.08957},
  year={2022}
}

@inproceedings{polosky2022constrained,
  title={Constrained offline policy optimization},
  author={Polosky, Nicholas and Da Silva, Bruno C and Fiterau, Madalina and Jagannath, Jithin},
  booktitle={International Conference on Machine Learning},
  pages={17801--17810},
  year={2022},
  organization={PMLR}
}

@inproceedings{guan2024poce,
  title={POCE: Primal Policy Optimization with Conservative Estimation for Multi-constraint Offline Reinforcement Learning},
  author={Guan, Jiayi and Shen, Li and Zhou, Ao and Li, Lusong and Hu, Han and He, Xiaodong and Chen, Guang and Jiang, Changjun},
  booktitle={Proceedings of the IEEE/CVF Conference on Computer Vision and Pattern Recognition},
  pages={26243--26253},
  year={2024}
}

@article{peng2021learning,
  title={Learning to simulate self-driven particles system with coordinated policy optimization},
  author={Peng, Zhenghao and Li, Quanyi and Hui, Ka Ming and Liu, Chunxiao and Zhou, Bolei},
  journal={Advances in Neural Information Processing Systems},
  volume={34},
  pages={10784--10797},
  year={2021}
}

@inproceedings{peng2022safe,
  title={Safe driving via expert guided policy optimization},
  author={Peng, Zhenghao and Li, Quanyi and Liu, Chunxiao and Zhou, Bolei},
  booktitle={Conference on Robot Learning},
  pages={1554--1563},
  year={2022},
  organization={PMLR}
}

@inproceedings{chen2018gradnorm,
  title={Gradnorm: Gradient normalization for adaptive loss balancing in deep multitask networks},
  author={Chen, Zhao and Badrinarayanan, Vijay and Lee, Chen-Yu and Rabinovich, Andrew},
  booktitle={International conference on machine learning},
  pages={794--803},
  year={2018},
  organization={PMLR}
}

@inproceedings{NEURIPS2020_4c5bcfec,
     author = {Ho, Jonathan and Jain, Ajay and Abbeel, Pieter},
     booktitle = {Advances in Neural Information Processing Systems},
     editor = {H. Larochelle and M. Ranzato and R. Hadsell and M.F. Balcan and H. Lin},
     pages = {6840--6851},
     publisher = {Curran Associates, Inc.},
     title = {Denoising Diffusion Probabilistic Models},
     url = {https://proceedings.neurips.cc/paper_files/paper/2020/file/4c5bcfec8584af0d967f1ab10179ca4b-Paper.pdf},
     volume = {33},
     year = {2020}
}

@inproceedings{li2021efficient,
  title={Efficient Learning of Safe Driving Policy via Human-AI Copilot Optimization},
  author={Li, Quanyi and Peng, Zhenghao and Zhou, Bolei},
  booktitle={International Conference on Learning Representations},
  year={2021}
}

@inproceedings{shridhar2023perceiver,
  title={Perceiver-actor: A multi-task transformer for robotic manipulation},
  author={Shridhar, Mohit and Manuelli, Lucas and Fox, Dieter},
  booktitle={Conference on Robot Learning},
  pages={785--799},
  year={2023},
  organization={PMLR}
}

@article{liang2022effective,
  title={Effective adaptation in multi-task co-training for unified autonomous driving},
  author={Liang, Xiwen and Wu, Yangxin and Han, Jianhua and Xu, Hang and Xu, Chunjing and Liang, Xiaodan},
  journal={Advances in Neural Information Processing Systems},
  volume={35},
  pages={19645--19658},
  year={2022}
}

@inproceedings{kalashnikov2021scaling,
  title={Scaling up multi-task robotic reinforcement learning},
  author={Kalashnikov, Dmitry and Varley, Jake and Chebotar, Yevgen and Swanson, Benjamin and Jonschkowski, Rico and Finn, Chelsea and Levine, Sergey and Hausman, Karol},
  booktitle={5th Annual Conference on Robot Learning},
  year={2021}
}

@article{lee2022multi,
  title={Multi-game decision transformers},
  author={Lee, Kuang-Huei and Nachum, Ofir and Yang, Mengjiao Sherry and Lee, Lisa and Freeman, Daniel and Guadarrama, Sergio and Fischer, Ian and Xu, Winnie and Jang, Eric and Michalewski, Henryk and others},
  journal={Advances in Neural Information Processing Systems},
  volume={35},
  pages={27921--27936},
  year={2022}
}

@inproceedings{d2019sharing,
  title={Sharing Knowledge in Multi-Task Deep Reinforcement Learning},
  author={D'Eramo, Carlo and Tateo, Davide and Bonarini, Andrea and Restelli, Marcello and Peters, Jan},
  booktitle={International Conference on Learning Representations},
  year={2019}
}

@inproceedings{sarafian2021recomposing,
  title={Recomposing the reinforcement learning building blocks with hypernetworks},
  author={Sarafian, Elad and Keynan, Shai and Kraus, Sarit},
  booktitle={International Conference on Machine Learning},
  pages={9301--9312},
  year={2021},
  organization={PMLR}
}

@article{ho2020denoising,
  title={Denoising diffusion probabilistic models},
  author={Ho, Jonathan and Jain, Ajay and Abbeel, Pieter},
  journal={Advances in neural information processing systems},
  volume={33},
  pages={6840--6851},
  year={2020}
}

@inproceedings{he2023diffusion,
    title={Diffusion Model is an Effective Planner and Data Synthesizer for Multi-Task Reinforcement Learning},
    author={Haoran He and Chenjia Bai and Kang Xu and Zhuoran Yang and Weinan Zhang and Dong Wang and Bin Zhao and Xuelong Li},
    booktitle={Thirty-seventh Conference on Neural Information Processing Systems},
    year={2023},
    url={https://openreview.net/forum?id=fAdMly4ki5}
}

@inproceedings{ajay2022conditional,
  title={Is Conditional Generative Modeling all you need for Decision-Making?},
  author={Ajay, Anurag and Du, Yilun and Gupta, Abhi and Tenenbaum, Joshua B and Jaakkola, Tommi S and Agrawal, Pulkit},
  booktitle={NeurIPS 2022 Foundation Models for Decision Making Workshop},
  year={2022}
}

@inproceedings{liu2023constrained,
  title={Constrained decision transformer for offline safe reinforcement learning},
  author={Liu, Zuxin and Guo, Zijian and Yao, Yihang and Cen, Zhepeng and Yu, Wenhao and Zhang, Tingnan and Zhao, Ding},
  booktitle={International Conference on Machine Learning},
  pages={21611--21630},
  year={2023},
  organization={PMLR}
}

@inproceedings{lee2021optidice,
  title={Optidice: Offline policy optimization via stationary distribution correction estimation},
  author={Lee, Jongmin and Jeon, Wonseok and Lee, Byungjun and Pineau, Joelle and Kim, Kee-Eung},
  booktitle={International Conference on Machine Learning},
  pages={6120--6130},
  year={2021},
  organization={PMLR}
}

@inproceedings{varghese2021optimization,
  title={Optimization of deep reinforcement learning with hybrid multi-task learning},
  author={Varghese, Nelson Vithayathil and Mahmoud, Qusay H},
  booktitle={2021 IEEE International Systems Conference (SysCon)},
  pages={1--8},
  year={2021},
  organization={IEEE}
}

@inproceedings{le2019batch,
  title={Batch policy learning under constraints},
  author={Le, Hoang and Voloshin, Cameron and Yue, Yisong},
  booktitle={International Conference on Machine Learning},
  pages={3703--3712},
  year={2019},
  organization={PMLR}
}

@inproceedings{lee2021coptidice,
  title={COptiDICE: Offline Constrained Reinforcement Learning via Stationary Distribution Correction Estimation},
  author={Lee, Jongmin and Paduraru, Cosmin and Mankowitz, Daniel J and Heess, Nicolas and Precup, Doina and Kim, Kee-Eung and Guez, Arthur},
  booktitle={International Conference on Learning Representations},
  year={2022}
}

@article{dai2020coindice,
  title={Coindice: Off-policy confidence interval estimation},
  author={Dai, Bo and Nachum, Ofir and Chow, Yinlam and Li, Lihong and Szepesv{\'a}ri, Csaba and Schuurmans, Dale},
  journal={Advances in neural information processing systems},
  volume={33},
  pages={9398--9411},
  year={2020}
}

@inproceedings{zhang2019gendice,
  title={GenDICE: Generalized Offline Estimation of Stationary Values},
  author={Zhang, Ruiyi and Dai, Bo and Li, Lihong and Schuurmans, Dale},
  booktitle={International Conference on Learning Representations},
  year={2019}
}

@article{liu2023datasets,
  title={Datasets and Benchmarks for Offline Safe Reinforcement Learning},
  author={Liu, Zuxin and Guo, Zijian and Lin, Haohong and Yao, Yihang and Zhu, Jiacheng and Cen, Zhepeng and Hu, Hanjiang and Yu, Wenhao and Zhang, Tingnan and Tan, Jie and others},
  journal={arXiv preprint arXiv:2306.09303},
  year={2023}
}

@inproceedings{songdenoising,
  title={Denoising Diffusion Implicit Models},
  author={Song, Jiaming and Meng, Chenlin and Ermon, Stefano},
  booktitle={International Conference on Learning Representations},
  year={2022}
}

@article{xu2022constraints,
  title={Constraints penalized q-learning for safe offline reinforcement learning},
  author={Xu, Haoran and Zhan, Xianyuan and Zhu, Xiangyu},
  journal={Proceedings of the AAAI Conference on Artificial Intelligence},
  volume={36},
  number={8},
  pages={8753--8760},
  year={2022}
}

@article{luo2022understanding,
  title={Understanding diffusion models: A unified perspective},
  author={Luo, Calvin},
  journal={arXiv preprint arXiv:2208.11970},
  year={2022}
}

@article{kumar2020conservative,
  title={Conservative q-learning for offline reinforcement learning},
  author={Kumar, Aviral and Zhou, Aurick and Tucker, George and Levine, Sergey},
  journal={Advances in Neural Information Processing Systems},
  volume={33},
  pages={1179--1191},
  year={2020}
}

@inproceedings{ishihara2021multi,
  title={Multi-task learning with attention for end-to-end autonomous driving},
  author={Ishihara, Keishi and Kanervisto, Anssi and Miura, Jun and Hautamaki, Ville},
  booktitle={Proceedings of the IEEE/CVF conference on computer vision and pattern recognition},
  pages={2902--2911},
  year={2021}
}

@inproceedings{ze2023gnfactor,
  title={Gnfactor: Multi-task real robot learning with generalizable neural feature fields},
  author={Ze, Yanjie and Yan, Ge and Wu, Yueh-Hua and Macaluso, Annabella and Ge, Yuying and Ye, Jianglong and Hansen, Nicklas and Li, Li Erran and Wang, Xiaolong},
  booktitle={Conference on Robot Learning},
  pages={284--301},
  year={2023},
  organization={PMLR}
}

@inproceedings{NEURIPS2022_f376f5df,
     author = {Yoo, Minjong and Cho, SangWoo and Woo, Honguk},
     booktitle = {Advances in Neural Information Processing Systems},
     editor = {S. Koyejo and S. Mohamed and A. Agarwal and D. Belgrave and K. Cho and A. Oh},
     pages = {37432--37444},
     publisher = {Curran Associates, Inc.},
     title = {Skills Regularized Task Decomposition for Multi-task Offline Reinforcement Learning},
     url = {https://proceedings.neurips.cc/paper_files/paper/2022/file/f376f5dff6f6ec6364aea7a46ab49574-Paper-Conference.pdf},
     volume = {35},
     year = {2022}
}

@inproceedings{yu2020meta,
  title={Meta-world: A benchmark and evaluation for multi-task and meta reinforcement learning},
  author={Yu, Tianhe and Quillen, Deirdre and He, Zhanpeng and Julian, Ryan and Hausman, Karol and Finn, Chelsea and Levine, Sergey},
  booktitle={Conference on robot learning},
  pages={1094--1100},
  year={2020},
  organization={PMLR}
}

@inproceedings{xie2022lifelong,
  title={Lifelong robotic reinforcement learning by retaining experiences},
  author={Xie, Annie and Finn, Chelsea},
  booktitle={Conference on Lifelong Learning Agents},
  pages={838--855},
  year={2022},
  organization={PMLR}
}

@inproceedings{NEURIPS2020_32cfdce9,
     author = {Yang, Ruihan and Xu, Huazhe and WU, YI and Wang, Xiaolong},
     booktitle = {Advances in Neural Information Processing Systems},
     pages = {4767--4777},
     title = {Multi-Task Reinforcement Learning with Soft Modularization},
     volume = {33},
     year = {2020}
}

@article{vithayathil2020survey,
  title={A survey of multi-task deep reinforcement learning},
  author={Vithayathil Varghese, Nelson and Mahmoud, Qusay H},
  journal={Electronics},
  volume={9},
  number={9},
  pages={1363},
  year={2020},
  publisher={MDPI}
}

@article{guan2023uac,
  title={UAC: Offline reinforcement learning with uncertain action constraint},
  author={Guan, Jiayi and Gu, Shangding and Li, Zhijun and Hou, Jing and Yang, Yiqin and Chen, Guang and Jiang, Changjun},
  journal={IEEE Transactions on Cognitive and Developmental Systems},
  volume={16},
  number={2},
  pages={671--680},
  year={2023},
  publisher={IEEE}
}

@article{kang2024efficient,
  title={Efficient diffusion policies for offline reinforcement learning},
  author={Kang, Bingyi and Ma, Xiao and Du, Chao and Pang, Tianyu and Yan, Shuicheng},
  journal={Advances in Neural Information Processing Systems},
  volume={36},
  year={2024}
}

@inproceedings{espeholt2018impala,
  title={Impala: Scalable distributed deep-rl with importance weighted actor-learner architectures},
  author={Espeholt, Lasse and Soyer, Hubert and Munos, Remi and Simonyan, Karen and Mnih, Vlad and Ward, Tom and Doron, Yotam and Firoiu, Vlad and Harley, Tim and Dunning, Iain and others},
  booktitle={International conference on machine learning},
  pages={1407--1416},
  year={2018},
  organization={PMLR}
}

@article{plappert2018multi,
  title={Multi-goal reinforcement learning: Challenging robotics environments and request for research},
  author={Plappert, Matthias and Andrychowicz, Marcin and Ray, Alex and McGrew, Bob and Baker, Bowen and Powell, Glenn and Schneider, Jonas and Tobin, Josh and Chociej, Maciek and Welinder, Peter and others},
  journal={arXiv preprint arXiv:1802.09464},
  year={2018}
}

@inproceedings{mnih2016asynchronous,
  title={Asynchronous methods for deep reinforcement learning},
  author={Mnih, Volodymyr and Badia, Adria Puigdomenech and Mirza, Mehdi and Graves, Alex and Lillicrap, Timothy and Harley, Tim and Silver, David and Kavukcuoglu, Koray},
  booktitle={International conference on machine learning},
  pages={1928--1937},
  year={2016},
  organization={PMLR}
}

@article{borsa2016learning,
  title={Learning shared representations in multi-task reinforcement learning},
  author={Borsa, Diana and Graepel, Thore and Shawe-Taylor, John},
  journal={arXiv preprint arXiv:1603.02041},
  year={2016}
}

@inproceedings{silver2014deterministic,
  title={Deterministic policy gradient algorithms},
  author={Silver, David and Lever, Guy and Heess, Nicolas and Degris, Thomas and Wierstra, Daan and Riedmiller, Martin},
  booktitle={International conference on machine learning},
  pages={387--395},
  year={2014},
  organization={Pmlr}
}

@article{blei2017variational,
  title={Variational inference: A review for statisticians},
  author={Blei, David M and Kucukelbir, Alp and McAuliffe, Jon D},
  journal={Journal of the American statistical Association},
  volume={112},
  number={518},
  pages={859--877},
  year={2017},
  publisher={Taylor \& Francis}
}

@inproceedings{deramo2020sharing,
  title={Sharing Knowledge in Multi-Task Deep Reinforcement Learning},
  author={DEramo, C and Tateo, D and Bonarini, A and Restelli, M and Peters, J},
  booktitle={Eighth International Conference on Learning Representations (ICLR 2020)},
  year={2020},
  organization={OpenReview. net}
}

@article{song2019generative,
  title={Generative modeling by estimating gradients of the data distribution},
  author={Song, Yang and Ermon, Stefano},
  journal={Advances in neural information processing systems},
  volume={32},
  year={2019}
}

@inproceedings{guan2023voce,
    title={{VOCE}: Variational Optimization with Conservative Estimation for Offline Safe Reinforcement Learning},
    author={Jiayi Guan and Guang Chen and Jiaming Ji and Long Yang and Ao Zhou and Zhijun Li and changjun jiang},
    booktitle={Thirty-seventh Conference on Neural Information Processing Systems},
    year={2023},
    url={https://openreview.net/forum?id=sIU3WujeSl}
}

@article{zhang2023saformer,
  title={Saformer: A conditional sequence modeling approach to offline safe reinforcement learning},
  author={Zhang, Qin and Zhang, Linrui and Xu, Haoran and Shen, Li and Wang, Bowen and Chang, Yongzhe and Wang, Xueqian and Yuan, Bo and Tao, Dacheng},
  journal={arXiv preprint arXiv:2301.12203},
  year={2023}
}

@book{cmdp,
  title={Constrained Markov decision processes: stochastic modeling},
  author={Altman, Eitan},
  year={1999},
  publisher={Routledge}
}
